%% file: main_arxiv.tex
\documentclass[format=acmsmall,nonacm]{acmart}
\usepackage{booktabs} 
\usepackage[ruled]{algorithm2e} 

\SetAlFnt{\small}
\SetAlCapFnt{\small}
\SetAlCapNameFnt{\small}
\SetAlCapHSkip{0pt}
\IncMargin{-\parindent}

\setcitestyle{authoryear}

\author{Yuxuan Lu}
\authornote{Both authors contributed equally to the paper and are listed in alphabetical order.}
\authornote{Supported by National Natural Science Foundation of China award number 62372007.}
\affiliation{\institution{Peking University}\city{Beijing}\country{China}}
\email{yx\_lu@pku.edu.cn}

\author{Shengwei Xu}
\authornotemark[1]
\authornote{Supported by United States National Science Foundation award number 2313137.}
\affiliation{\institution{University of Michigan}\city{Ann Arbor}\country{USA}}
\email{shengwei@umich.edu}

\author{Yichi Zhang}
\authornotemark[3]
\affiliation{\institution{University of Michigan}\city{Ann Arbor}\country{USA}}
\email{yichiz@umich.edu}

\author{Yuqing Kong}
\authornotemark[2]
\affiliation{\institution{Peking University}\city{Beijing}\country{China}}
\email{yuqing.kong@pku.edu.cn}

\author{Grant Schoenebeck}
\authornotemark[3]
\affiliation{\institution{University of Michigan}\city{Ann Arbor}\country{USA}}
\email{schoeneb@umich.edu}

\usepackage{hyperref}       
\usepackage{url}            
\usepackage{booktabs}       
\usepackage{amsfonts}       
\usepackage{nicefrac}       
\usepackage{multirow}
\usepackage{amstext}
\usepackage{amsmath}
\usepackage{amsthm}
\usepackage{bm}
\usepackage{wrapfig}
\usepackage{graphicx}
\usepackage{bbm, comment}
\usepackage{color}
\usepackage{float,subfigure}
\usepackage{wrapfig}
\usepackage{url}
\usepackage{autobreak}
\usepackage{threeparttable}
\usepackage[utf8]{inputenc}
\usepackage{graphicx} 
\usepackage{float} 
\usepackage{pgfplots}
\usepackage{caption}
\usepackage{graphicx}
\usepackage{cleveref}
\usepackage{mdframed}
\usepackage{xcolor}
\usepackage[most]{tcolorbox}
\tcbuselibrary{skins}
\usepackage{thmtools}
\usepackage{array}

\newcommand{\LLM}{\operatorname{LLM}}

\newcommand{\signal}{x}
\newcommand{\Signal}{X}
\newcommand{\report}{\tilde{x}}
\newcommand{\Report}{\tilde{X}}
\newcommand{\strategy}{\sigma}
\newcommand{\itemz}{z}
\newcommand{\Itemz}{Z}

\newcommand{\synopsis}{\theta}
\newcommand{\Synopsis}{\Theta}

\newcommand{\lowsig}{x^{l}}
\newcommand{\Lowsig}{X^{l}}

\newtheorem*{theorem*}{Theorem}
\newtheorem{theorem}{Theorem}[section]

\newtheorem{assumption}[theorem]{Assumption}

\newtheorem{proposition}[theorem]{Proposition}
\newtheorem{definition}[theorem]{Definition}

\usepackage{color-edits}
\addauthor[Yichi]{yichi}{violet}

\title{Eliciting Informative Text Evaluations with Large Language Models} 

\begin{abstract}
Peer prediction mechanisms motivate high-quality feedback with provable guarantees. However, current methods only apply to rather simple reports, like multiple-choice or scalar numbers. We aim to broaden these techniques to the larger domain of text-based reports, drawing on the recent developments in large language models. This vastly increases the applicability of peer prediction mechanisms as textual feedback is the norm in a large variety of feedback channels: peer reviews, e-commerce customer reviews, and comments on social media.

We introduce two mechanisms, the \textsc{Generative Peer Prediction Mechanism (GPPM)} and the \textsc{Generative Synopsis Peer Prediction Mechanism (GSPPM)}. These mechanisms utilize LLMs as predictors, mapping from one agent's report to a prediction of her peer's report. Theoretically, we show that when the LLM prediction is sufficiently accurate, our mechanisms can incentivize high effort and truth-telling as an (approximate) Bayesian Nash equilibrium. Empirically, we confirm the efficacy of our mechanisms through experiments conducted on two real datasets: the Yelp review dataset and the ICLR OpenReview dataset. We highlight the results that on the ICLR dataset, our mechanisms can differentiate three quality levels --- human-written reviews, GPT-4-generated reviews,  and GPT-3.5-generated reviews in terms of expected scores. Additionally, GSPPM penalizes LLM-generated reviews more effectively than GPPM.

\end{abstract}

\begin{document}

\begin{titlepage}

\maketitle
\makeatletter \gdef\@ACM@checkaffil{} \makeatother


\end{titlepage}

\input{new_intro}

\section{Related Work}

\paragraph{Peer prediction mechanisms}

The idea of peer prediction, introduced by \citet{miller2005eliciting}, proposes a mechanism that induces truth-telling as an equilibrium, assuming the mechanism knows the joint distribution of agents' signals and this knowledge is common knowledge among agents. This fundamental assumption, however, might be challenged in many settings. There have been two primary research directions that can mitigate it.

One primary direction is to generalize peer prediction to a multi-task setting, making it possible to estimate (implicitly or explicitly) the joint distribution between the reports with multiple tasks \cite{dasgupta2013crowdsourced,shnayder2016informed, kong2019information, schoenebeck2020learning,10.1145/3543507.3583292,kong2020dominantly,Agarwal2017-ty,rowdycrowds}. The other direction follows another mechanism, the Bayesian Truth Serum (BTS) \cite{prelec2004bayesian}, which also seeks to incentivize truthful reporting as an equilibrium. It directly solicits predictions from agents, i.e., each agent predicts the reports of their peers, bypassing the need for prior knowledge of distributions. Subsequent research has extended the BTS, enhancing its efficacy in diverse scenarios \cite{10.5555/2900929.2900940, 10.5555/2891460.2891576,zhang2014elicitability,10.1145/3565560}.

Our contribution markedly diverges from these established approaches in several ways. First, we focus on incentivizing text-based human responses, which entails a highly complex response space, where both estimating the joint distribution and directly eliciting predictions from agents are no longer practical. Second, we introduce the first peer prediction mechanism that employs a Large Language Model (LLM) as an oracle, which is presumed to effectively approximate the joint distribution of responses. As LLM capabilities enhance, our mechanism naturally improves, benefiting from more accurate oracle estimations of joint response distributions. Overall, our approach enables the deployment of peer prediction mechanisms in single-task settings without the necessity of directly soliciting predictions or possessing priori distribution knowledge.

\paragraph{Shortcuts in information elicitation}

\citet{geirhos2020shortcut} use ``shortcut learning'' to describe a machine learning algorithm's reliance on unintended, simpler features rather than learning the intended, complex representations. For example, a deep neural network might only identify cows accurately within grassy scenes, mistaking ``grass'' for ``cows'' \cite{beery2018recognition}. \citet{niven2019probing} show that a language model, BERT \cite{devlin2018bert}, might seem to understand reasoning but actually relies on spurious statistical cues in the dataset, particularly the presence of specific words, failing when these cues are absent.

This phenomenon not only exists in artificial intelligent agents but also in human agents. As accessing and processing the information may cost agents a considerable amount of effort, human agents also rely on shortcuts or heuristics. A case in point is peer review processes, where reviewers might focus primarily on assessing the story-telling in the introduction as a shortcut, instead of providing a deeper, more insightful evaluation of the work's merits.

To address this issue and elicit high-quality information rather than superficial shortcuts, researchers have proposed several approaches. \citet{li2022optimization} propose optimal scoring rules that maximize the marginal expected payoff increase of an agent who exerts a higher level of effort. \citet{zhang2022high} suggest that paying the agents according to a tournament based on their scores could also help the efficiency to elicit high effort. When applying peer prediction mechanisms in a peer grading setting, \citet{gao2016incentivizing} finds a Pareto-optimal Nash equilibrium where all agents exploit a low-effort shortcut. To address this problem, \citet{kong2018eliciting} propose a peer prediction mechanism robust against such shortcuts, by explicitly eliciting both the intended response and the shortcut, and then conditioning the shortcut out. Our GSPPM is inspired by \citet{kong2018eliciting} but focuses on text-based responses and does not require eliciting additional information from the agents.

Additionally, how to robustly aggregate information when there exists (implicitly or explicitly) a shortcut is also widely studied \cite{prelec2017solution,pan2023robust,chen2021wisdom,guo2024algorithmic,arieli2018robust,palley2019extracting,lu2024calibrating,kong2022eliciting}. A well-known task is ``Is Philadelphia the capital of Pennsylvania?'' where the majority will take the shortcut that Philadelphia is a famous big city and big cities are usually the capital, thus making mistakes in this question. \citet{prelec2017solution} provide a solution to this by eliciting the agents' prediction and aggregating the ``surprisingly popular'' choice.

\paragraph{LLM applications in peer review}

The core of our experiments relies on using LLMs to analyze peer review data on ICLR and customer review data on Yelp.
Recent works claim that LLMs have great potential to support publication \cite{gao2022comparing, salvagno2023can} and review practices \cite{liang2023large} if properly used. In a survey, it is shown that GPT-generated reviews are appreciated by over 50\% of participants, and over 80\% of participants found them more beneficial than feedback from at least some human reviewers \cite{liang2023large}. More evidence exists showing that LLMs can be used to handle easy evaluation tasks in peer review, such as summarizing methodologies and assessing writing styles \cite{donker2023dangers}, detecting gender bias \cite{10.7554/eLife.90230}, and verifying author checklists \cite{liu2023reviewergpt}. However, other work expresses concerns about the ability of LLMs to handle more complicated reviewing tasks. For example, hallucinating occurs when LLMs are asked to provide suggestions for specific improvements~\cite{donker2023dangers}, and LLMs tend to perform poorly on distinguishing clearly ``better'' papers \cite{liu2023reviewergpt}. \citet{Hosseini23} provide some recommendations for the proper use of LLMs to support review or editorial processes.

Our work is largely orthogonal to these concerns. Rather than using LLMs to generate reviews directly, our paper focuses on using LLMs as a tool to design automated mechanisms that incentivize high-quality reviews. If the quality of LLM reviews dominated the quality of high-effort human reviews, our work would largely be superfluous, but LLMs cannot yet do this. While LLM-generated reviews may provide additional feedback, for the moment, subjective high-quality human reviews are still essential.

\section{Preliminaries}

This section introduces the classic information elicitation model and the preliminaries that guide the design of our method.

\subsection{Model} \label{subsec:model}
In our setting, a set of items (e.g., papers or restaurants) are reviewed by a set of agents, where each item is assigned to multiple agents for review. We reduce the problem to the setting where there is only one item to be reviewed by two agents and emphasize that our method can be applied to any item and any pair of agents. Let $I = \{1,2\}$ be the set of agents reviewing the same item.

\paragraph{Item and Signal} Agents' judgments of the item are influenced by the inherent characteristics of the item and other related background knowledge used to generate the judgment. Let $\Itemz\in \mathcal{Z}$ denote an item such that observing $\Itemz=\itemz$ is sufficient for an agent to form her judgments about the item. Suppose $\Itemz$ is sampled from an unknown common prior $\pi$.

\begin{figure}[!ht]
    \centering
    \includegraphics[width=0.6\linewidth]{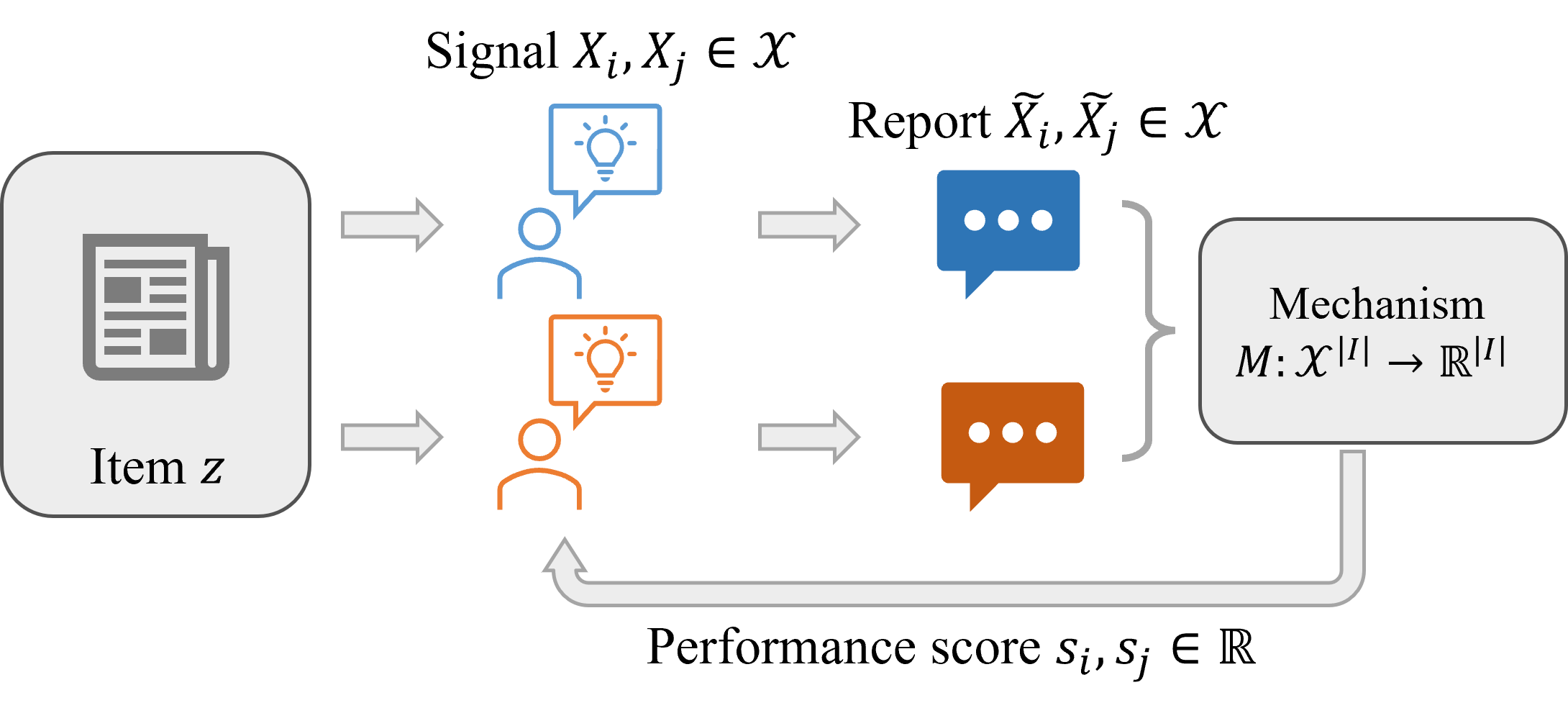}
    \caption{An Overview of Our Information Elicitation Model}
    \label{fig:model}
\end{figure}

Given an item $\Itemz=\itemz$, each agent receives a subjective signal $\Signal_i\in \mathcal{X}$ when evaluating the item. We use $\signal_i$ to denote a potential value for $\Signal_i$. Similar to prior literature \cite{miller2005eliciting, dasgupta2013crowdsourced}, we adopt the common assumption that signals are i.i.d.~conditioned on the item, i.e., $\Pr[\Signal_i\mid \Itemz=\itemz]$ is identical for any agent $i$. 
Furthermore, when agents evaluate the same item, their signals are expected to be related in some meaningful way. We thus adopt the following assumption, which is required to guarantee that truthful information is elicitable~\cite{Cremer_Mclean}.

\begin{assumption}[Stochastic Relevance] 
For any $\signal_i, \signal_i' \in \mathcal{X}$ such that $\signal_i\neq\signal_i'$, there exists $\signal_j\in \mathcal{X}$ such that
\[\Pr[\Signal_j = \signal_j \mid  \Signal_i = \signal_i] \neq \Pr[\Signal_j = \signal_j \mid  \Signal_i = \signal_i'].\]
\end{assumption}

\paragraph{Reporting Strategy} 
We assume the report of agent $i$, denoted as $\Report_i\in \mathcal{X}$, shares the same domain as the signal. Agents can truthfully report their signals or manipulate their signals as reports. 
However, only reports (not signals) are observable to the mechanism. 
Let $\sigma_i: \mathcal{X}\rightarrow\Delta_\mathcal{X}$ denote the reporting strategy of agent $i$ so that $\Report_i = \sigma_i(X_i)$. Let $\tau$ denote the truthful reporting strategy such that $\tau(X_i) = X_i$ for any signal.

\paragraph{Effort Model}

\begin{figure}[!ht]
    \centering
    \includegraphics[width=0.6\linewidth]{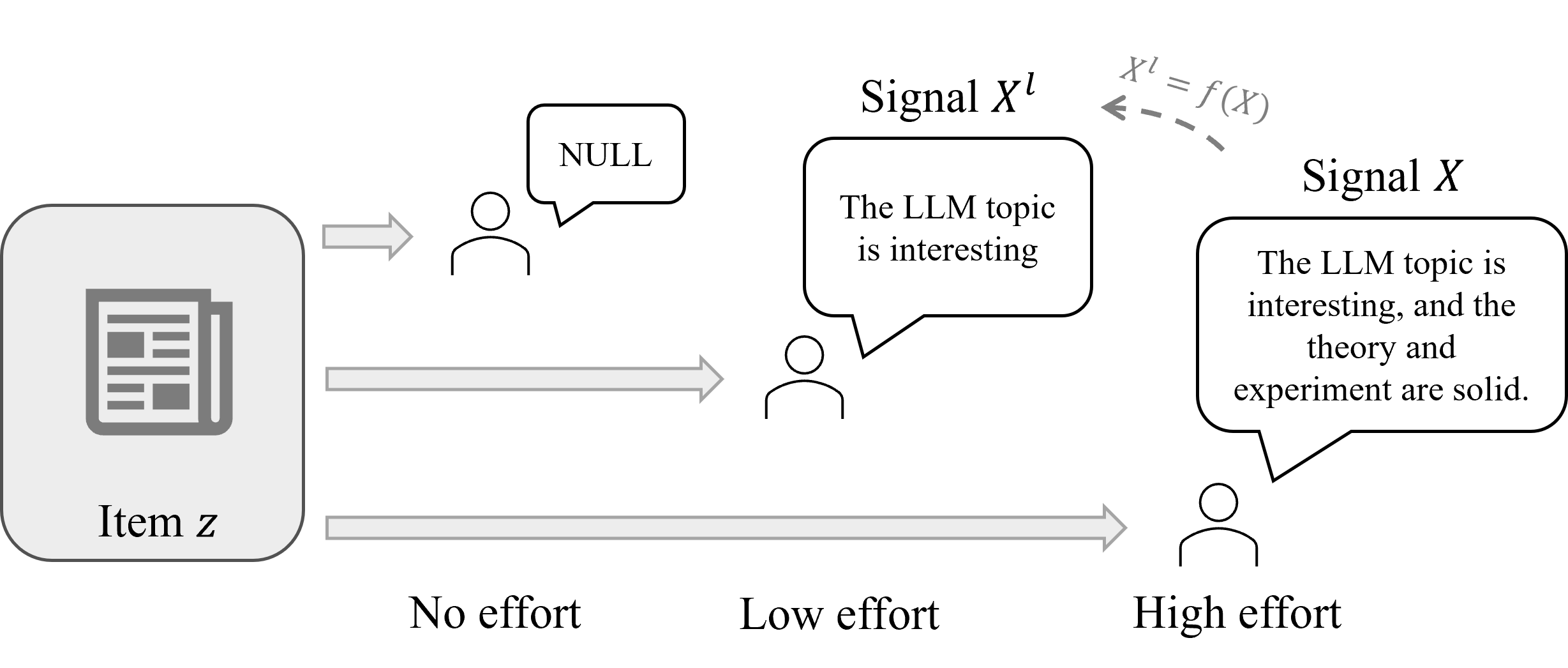}
    \caption{An Example of the Three-Level Effort Model}
    \label{fig:effort-model}
\end{figure}

We assume that agents can obtain signals with different qualities if they exert different levels of effort. We focus on a three-level effort model: exerting no effort with a cost of 0, low effort with a cost of $c_l\ge 0$, and high effort with a cost of $c_h>c_l$. We use $c_i$ to denote the agent $i$'s cost.

For each agent $i$, if she exerts high effort, she observes a high-quality signal $X_i=x_i$. If she exerts low effort, she observes a low-quality signal $\Lowsig_i=f(\signal_i)$, which is part of the high-quality signal. Here, $f$ is a deterministic function, so $\Pr[\Lowsig_i = f(x_i) \mid X_i = x_i] = 1$. However, given $\Lowsig_i$, the high-quality signal $X_i$ is uncertain. If an agent exerts no effort, they observe an uninformative signal $\textsc{Null}$ all the time. In this case, $\Pr[X = x \mid \textsc{Null}] = \Pr[X = x]$ for all random variables $X$. The above model implies that the conditional entropy\footnote{The conditional entropy $H(X \mid Y)$ measures the average amount of uncertainty in $X\in \mathcal{X}$ given the value of another random variable $Y\in \mathcal{Y}$, i.e., $H(X \mid Y) = - \sum_{x \in \mathcal{X}} \sum_{y \in \mathcal{Y}} P(x,y) \log \left( \frac{P(x,y)}{P(y)} \right)$.} $H(\Lowsig_i \mid X_i)=0$, and $H(X_i \mid \Lowsig_i)\geq 0$.

This effort model, called the hierarchical effort model \cite{kong2018eliciting}, suggests that the high-effort signals are strictly more informative than the low-effort signals. For example, in the setting of peer review, a low-effort reviewer might focus solely on surface-level aspects such as the writing quality of the paper. In contrast, a high-effort reviewer can assess not only the writing quality but also other aspects, such as the novelty of the idea, the soundness of the method, and the validity of the experiments.

Note that our model can also capture a binary effort setting which is assumed in some previous work \cite{miller2005eliciting,gao2016incentivizing}, by assuming $\Lowsig_i = f(x_i) = \textsc{Null}$ for any $x_i\in\mathcal{X}$. In this setting, the agent either receives an informative signal or a $\textsc{Null}$ signal.

\paragraph{Synopsis} 
In some applications, we are able to find a commonly known synopsis $\Synopsis = \synopsis$ of the item $z$. We use $\Sigma$ to denote the space of the synopsis. The synopsis may determine the low-effort signals. For example, in peer review, when a low-effort reviewer writes a review based solely on the introduction of the paper, the paper introduction can be regarded as a commonly known synopsis that fully determines the low-effort signal. 

\begin{definition}[Synopsis-determined Low-effort Signals]\label{def:synopsis-determined}
We say the low-effort signal is synopsis-determined if $\Lowsig_i = g(\synopsis)$ where $g$ is a deterministic function that maps the synopsis to the low-effort signal (Figure~\ref{fig:synopsis}).
\end{definition}

\begin{figure}[!ht]
    \centering
    \includegraphics[width=0.54\linewidth]{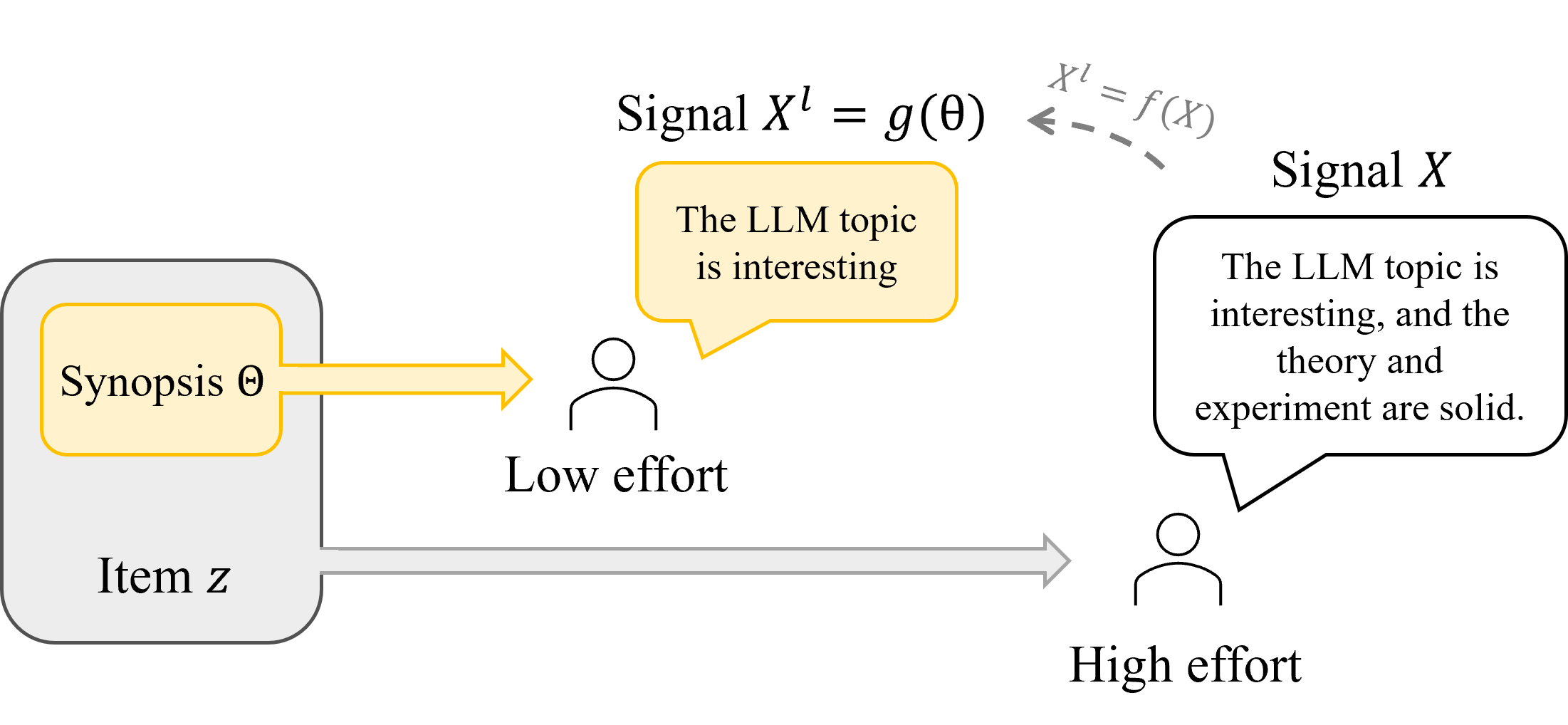}
    \caption{An Example of the Synopsis-determined Low-effort Signals}
    \label{fig:synopsis}
\end{figure}

In some cases, the low-effort signal $\Lowsig_i$ may contain more information about $X_i$ than the synopsis $\synopsis$, and the synopsis $\synopsis$ does not reveal more information about $X_i$ than $\Lowsig_i$. In the peer review example, a low-effort reviewer may write a review based on the whole introduction and a glance at the rest of the paper.

\begin{definition}[Synopsis-covering Low-effort Signals]\label{def:synopsis-covering}
We say the low-effort signal is synopsis-covering if $\Pr[X_i = x \mid \Lowsig_i = \lowsig_i, \Synopsis = \synopsis] = \Pr[X_i = x \mid \Lowsig_i = \lowsig_i]$ for any $x\in \mathcal{X}$.
\end{definition}

Note that Synopsis-determined and Synopsis-covering are not mutually exclusive. We will use these definitions in Corollary~\ref{coro:gap_bound}, but Propositions~\ref{prop:GPPM}, \ref{prop:GSPPM}, \ref{prop:GPPM-potent} and \ref{prop:GSPPM-potent} do not require them.

\paragraph{Peer Prediction Mechanism and Agents' Incentive} 
A peer prediction mechanism $M: \mathcal{X}^{|I|}\rightarrow \mathbb{R}^{|I|}$ takes all agents' reports as input and outputs a \emph{performance score} $s_i$ to each agent $i$, i.e., $s_i = M(\Report_i, \Report_j)$. 
Then, agent $i$ is paid according to a linear function $p_i = \alpha \cdot s_i + \beta$, where $\alpha >0$ and $\beta$ are constant parameters. We assume that the utility of agent $i$ is the difference between her payment and her cost of effort, i.e., $u_i = \alpha \cdot s_i + \beta - c_i$.

Given an item and a mechanism, each agent aims to maximize her expected utility by choosing a reporting strategy $\sigma_i$ and an effort $c_i$. Note that we focus on pure effort strategy, but the reporting strategy $\sigma_i$ can be randomized. Let $(\sigma_i, c_i)$ be agent $i$'s strategy and let $\{(\sigma_i,c_i)\}_{i\in I}$ be the strategy profile of all agents. Under a peer prediction mechanism, an agent's performance score and utility depend on the other agents' strategy. Therefore, we sometimes write agent $i$'s expected utility as a function of the strategy profile of both agents, i.e., $U_i((\sigma_i, c_i), (\sigma_j, c_j))$.

\subsection{Mechanism Design Goal}

We first define the solution concept. 

\begin{definition}
    A strategy profile $\{(\sigma_i,c_i)\}_{i\in I}$ is an $\epsilon$-BNE if, for any agent $i$ and for any alternative strategies $(\sigma_i', c_i')$, we have:
\[U_i((\sigma_i, c_i), (\sigma_j, c_j)) \ge U_i((\sigma_i', c_i'), (\sigma_j, c_j)) - \epsilon.\]
\end{definition}
In other words, no agent can gain more than $\epsilon$ in expected utility by unilaterally deviating from her strategy in an $\epsilon$-BNE.

At a high level, our goal is to design a mechanism that (approximately) maximizes an agent's expected performance score if she exerts high effort and reports truthfully. We call such a mechanism ($\epsilon$-)\emph{potent}.

\begin{definition}[Potent Mechanism]\label{def:potent-mechanism}
A peer prediction mechanism $M$ is $\epsilon$-potent if there exists a linear payment scheme with parameters $\alpha,\beta$, such that exerting high effort ($c_i = c_h$) and reporting truthfully ($\sigma_i=\tau$) is an $\epsilon$-Bayesian Nash equilibrium.
\end{definition}

\subsection{\citet{miller2005eliciting}'s Peer Prediction Mechanism}

\citet{miller2005eliciting} propose the first peer prediction mechanism. We refer to this as the original peer prediction mechanism. The original mechanism scores an agent based on how well her report predicts a randomly selected peer's report. They use the log scoring rule (LSR) \cite{cooke1991experts} to quantify the quality of the prediction.

\begin{definition}[Log Scoring Rule (LSR)]

Given a set of outcomes, $\mathcal{Y}$, and a prediction over the outcomes $p\in \Delta \mathcal{Y}$, the log scoring rule maps the prediction and an outcome $y\in \mathcal{Y}$ to a score $\operatorname{LSR}(p,y) = \log (p(y))$. Furthermore, for $q \in \Delta \mathcal{Y}$, let $\operatorname{LSR}(p,q) = \sum_{y\in\mathcal{Y}} q(y) \operatorname{LSR}(p,y)$ denote the expected log score when the outcome is sampled from the distribution $q$.

\end{definition}

The Log Scoring Rule is \textit{strictly proper} \cite{selten1998axiomatic,gneiting2007strictly}, i.e., \[\operatorname{LSR}(p,q) < \operatorname{LSR}(p,p),~ \forall p,q\in\Delta\mathcal{Y} \text{ and } p\neq q,\] meaning that reporting the true belief of the outcome maximized the score. The idea of \citet{miller2005eliciting}'s peer prediction mechanism is thus very straightforward: scoring agent $i$ based on how well her report can predict her peer's report according to a proper scoring rule (such as the Log Scoring Rule).

\begin{definition}[Original Peer Prediction Mechanism]\label{def:miller}
    Given agent $i$'s report $\report_i$, and the peer agent $j$'s report $\report_j$. The performance score of agent $i$ is 
    \[\operatorname{LSR}(\Pr[\Signal_j \mid \Signal_i = \report_i], \report_j) = \log \Pr[\Signal_j = \report_j \mid \Signal_i = \report_i].\]
\end{definition}

\begin{proposition}[Proposition 1 of \cite{miller2005eliciting}]
   In the binary effort model, if the common prior $\pi$ and $\Pr[X_i\mid Z=z]$ are known, the above mechanism is potent.
\end{proposition}
Intuitively, this mechanism incentivizes effort and truth-telling because only the prediction induced by the informative true signal can maximize the expected score under a strictly proper scoring rule.

\section{Peer Prediction Mechanisms for Textual Signals}
\label{sec:mechanisms}

As discussed, the original peer prediction mechanism requires knowledge about the prior distribution over the items and signals. However, the conditional distribution over signals $\Pr[\Signal_j = \signal_j \mid \Signal_i = \signal_i]$ can often be complex and difficult to learn from historical data, especially when eliciting textual signals. To address the textual settings, we propose leveraging Large Language Models (LLMs) to create an estimator of this distribution. We design two new mechanisms --- the \textsc{Generative Peer Prediction Mechanism (GPPM)}  and \textsc{Generative Synopsis Peer Prediction Mechanism (GSPPM)}. Assuming a sufficiently accurate LLM estimator, we theoretically prove that both mechanisms are (approximately) potent and show that GSPPM can better differentiate exerting high effort from exerting low effort by reducing the difference between exerting low effort and exerting no effort.

\subsection{Main Assumption: the LLM-Prediction}
We first integrate LLMs into the information elicitation model and discuss the key assumption that builds the theoretical foundations of our mechanisms.
Given a prompt $\psi$ and a response $\report$, a pre-trained LLM can produce a prediction indicating the likelihood of the response being $\report$. 
We denote the distribution of responses generated by an LLM with a prompt $\psi$ as $\Pr_{\operatorname{LLM}(\psi)}$ and refer to it as the \textit{LLM-prediction}. 
Thus, $\Pr_{\LLM(\psi)}[\report]$ denotes the probability that response $\report$ is predicted by $\operatorname{LLM}$ via prompt $\psi$. 

Sometimes, the prompt itself depends on some input $y$ (e.g., the review from a different agent), in which case we write the prompt as $\psi(y)$. If the input to the prompt itself is a random variable $Y$
, (which it often is as in Figure~\ref{fig:prompt-eg}), 
to be consistent with the classic information elicitation model, we use $\Pr_{\LLM(\psi)}[\Signal_j = \cdot\mid Y = y]$ to denote the LLM-prediction $\Pr_{\LLM(\psi(y))}[\cdot]$. We are particularly interested in two predictions. First, $\Pr_{\LLM(\psi)}[\Signal_j = \report_j \mid \Signal_i = \report_i]$ is the LLM-prediction of agent $j$'s report while integrating agent $i$'s report into the prompt. Second, $\Pr_{\LLM(\psi)}[\Signal_j = \report_j \mid \Signal_i = \report_i, \Synopsis = \synopsis]$ is the LLM-prediction of agent $j$'s report given agent $i$'s report $\report_i$ and a commonly known synopsis $\synopsis$. 
 
Our theoretical results require the following (strong) assumption about the fidelity of the LLM-prediction to that of the model:

\begin{assumption}[LLM-Prediction]\label{assum:llm-prior}
    We assume that for an information elicitation task of interest, there exist sets of prompts $\{\psi(\signal)\}_{\signal \in \mathcal{X}}$, 
    $\{\psi'(\signal, \synopsis)\}_{\signal \in \mathcal{X},\synopsis \in \Sigma}$, and $\epsilon, \epsilon' \ge 0$ such that for any pair of signal $\signal_i\in \mathcal{X}$ and synopsis $\synopsis\in \Sigma$:  
    \[D_{KL}\left[\Pr[\Signal_j = \cdot \mid \Signal_i = \signal_i] ~\Big\|~ \Pr_{\LLM(\psi)}[\Signal_j = \cdot \mid \Signal_i = \signal_i]  \right] \le \epsilon,\] 
    \[D_{KL}\left[\Pr[\Signal_j = \cdot \mid \Signal_i = \signal_i, \Synopsis = \synopsis] ~\Big\|~ \Pr_{\LLM(\psi')}[\Signal_j = \cdot \mid \Signal_i = \signal_i, \Synopsis = \synopsis]  \right] \le \epsilon',\] 
    where $D_{KL}[P~\|~Q]$ denotes the KL-divergence\footnote{The KL-divergence between two distributions over the same probability space is $D_{\text{KL}}(P \| Q) = \sum_{x} P(x) \log \left(P(x)/Q(x)\right).$} between two distributions. Furthermore, this fact is common knowledge for all agents. 
\end{assumption}

\Cref{assum:llm-prior} implies that the prediction of an (idealized) LLM can accurately estimate the underlying information structure of the high-effort signals, allowing us to leverage the vast knowledge embedded within LLMs to predict the probability of a new review $\signal_j$ given an existing review $\signal_i$. Such an LLM provides a data-driven way of computing this distribution. In \Cref{sec:method}, we will detail various \emph{implementations} for computing an approximation to $\Pr_{\LLM(\psi(\signal_i))}$.

We now turn toward defining our two mechanisms by assuming that we have access to some way of computing the LLM-prediction.

\subsection{The Generative Peer Prediction Mechanism (GPPM)}

To define our \textsc{Generative Peer Prediction Mechanism (GPPM)}, we combine the idea of an LLM prediction with Miller et al.'s mechanism in the textual setting.

\begin{definition}[\textsc{Generative Peer Prediction Mechanism (GPPM)}]\label{def:GPPM} 
    Given the peer's report $\report_j$, the performance score of agent $i$ with report $\report_i$ is 
    \[\operatorname{LSR}(\Pr_{\LLM(\psi)}[\Signal_j \mid \Signal_i = \report_i], \report_j) = \log \Pr_{\LLM(\psi)}[\Signal_j = \report_j \mid \Signal_i = \report_i]\]
\end{definition}

If \Cref{assum:llm-prior} holds and the peer's report $\report_j$ is of high effort and truthful, \textsc{GPPM} will be able to successfully differentiate between different effort levels. We defer the formal theory to \Cref{subsec:theory} while providing the high-level intuitions of why \textsc{GPPM} is potent here.

First, GPPM should be able to reward an effortful signal more than a no-effort signal. This is because even a low-effort signal $\Signal^l_i$ helps predict some of the terms in $\Signal_j$. 
For example, if $\Signal^l_i$ merely makes it clear that the reviewed paper proposes a new machine learning algorithm, then the words ``loss function'' and ``training'' will (correctly) be much more likely.

In addition, we anticipate that the GPPM should be able to reward a high-effort signal more than a low-effort signal. Intuitively, this is because some insights/critiques can only be predicted by high-effort signals.  
Consider the peer review example again, a low-effort signal may report common features of a machine learning paper but could overlook specific details of the particular paper such as an elegant proof and a potential broader impact of the method.

\subsection{The Generative Synopsis Peer Prediction Mechanism (GSPPM)}

Intuitively, the above \textsc{GPPM} pays both high and low-effort signals. However, in some applications where low-effort signals may be easily generated by LLMs, such as academic peer review shown in \cref{fig:example-llm-vs-human}, we only want the high-effort signals. This raises the question: Is it possible to further penalize the reporting of low-effort cheap signals? 

We propose the \textsc{Generative Synopsis Peer Prediction Mechanism (GSPPM)}. The idea is to condition the LLM-prediction on the commonly known synopsis, such as a summary of the reviewed item. 
As the low-effort signal can be fully characterized by the synopsis, reporting the low-effort signal brings no extra information on predicting agent $j$'s report when the synopsis is conditioned out.
In the above peer review example, if the abstract of the reviewed paper has already been inputted into the LLM as prompts, then a low-effort signal simply reiterating the abstract's contents would be redundant.

\begin{definition}[Generative Synopsis Peer Prediction Mechanism (GSPPM)]\label{def:GSPPM}
Suppose we have a synopsis $\synopsis$ of the item being reviewed. Given the peer's report being $\report_j$, the performance score for agent $i$ with report $\report_i$ is
   \[\operatorname{LSR}(\Pr_{\LLM}[\Signal_j \mid \Signal_i = \report_i, \Synopsis = \synopsis], \report_j) = \log \Pr_{\LLM}[\Signal_j = \report_j \mid \Signal_i = \report_i, \Synopsis = \synopsis].\]
\end{definition}

In practice, we hypothesize that GSPPM can generally decrease the scores of low-effort signals to that of no-effort signals. This is because the low-effort signal is unlikely to offer additional insights beyond those already present in the synopsis.  

For the same reason, we hypothesize that GSPPM can outperform GPPM in distinguishing between low-effort and high-effort signals. Because the scores of low-effort signals are pushed closer to the baseline of no-effort, the reduction in entropy by providing more specific insights that appear in both $\Signal_i$ and  $\Signal_j$ should be more salient. 
In other words, we hypothesize that it will improve the signal-to-noise ratio by making insights more prominent and vocabulary alignment less important.  

This intuition is very similar to that in ~\citet{kong2018eliciting}, which also used conditioning to motivate high-effort signals above low-effort signals.

\subsection{Theoretical Results: GPPM and GSPPM are $\varepsilon$-Potent} \label{subsec:theory}

Here, we provide formal theoretical guarantees of our mechanisms under \Cref{assum:llm-prior}. We first present several important notations before introducing our propositions.

We use $I(X_i;X_j)$ to denote the Shannon mutual information \cite{shannon1948mathematical} between two signals, \( X_i \) and \( X_j \). It provides a quantitative measure of the information shared between them.
\begin{small}\begin{align*}
    I(X_i;X_j) =& \sum_{x_i,x_j \in \mathcal{X}} \Pr[X_i=x_i, X_j=x_j] \log \frac{\Pr[X_i=x_i, X_j=x_j]}{\Pr[X_i=x_i]\Pr[X_j=x_j]} 
\end{align*}\end{small}
Furthermore, we use \( I(X_i;X_j \mid X_i^l) \) to denote the conditional mutual information.
\begin{small}\begin{align*}
& I(X_i;X_j \mid X_i^l) \\
= & \sum_{X^l_i \in \mathcal{X}} \Pr[X^l_i = x^l_i] \sum_{x_i,x_j \in \mathcal{X}} \Pr[X_i=x_i, X_j=x_j \mid X^l_i = x^l_i] \log \frac{\Pr[X_j=x_j,X_i=x_i \mid  X^l_i = x^l_i]}{\Pr[X_i=x_i \mid X^l_i = x^l_i]\Pr[X_j=x_j \mid X^l_i = x^l_i]}
\end{align*}\end{small}

Then, we have the following propositions for GPPM and GSPPM respectively.

\begin{restatable}{proposition}{propGPPM}\label{prop:GPPM}
For the \textsc{Generative Peer Prediction Mechanism (GPPM)}, when \Cref{assum:llm-prior} holds with parameter $\epsilon\geq0$, for any agent $i$, given the peer agent $j$ exerting high effort and reporting truthfully, any untruthful reporting strategy $\hat{\sigma}\neq \tau$ or effort $c_i\in\{0,c_l\}$ (no-effort or low-effort) implying signal $\textsc{Null}$ or $\Lowsig_i$ won't bring more than $\epsilon$ score increase. Specifically, we have
\begin{align*}
    \mathbb{E}[\textsc{GPPM}(\tau(\Signal_i), \Signal_j)] - \mathbb{E}[\textsc{GPPM}(\hat{\sigma}(\Signal_i), \Signal_j)] &> - \epsilon. \tag{High effort, truthful v.s. High effort, untruthful}\\
    \mathbb{E}[\textsc{GPPM}(\tau(\Signal_i), \Signal_j)] - \mathbb{E}[\textsc{GPPM}(\strategy(\Signal^l_i), \Signal_j)] &\geq I(\Signal_i;\Signal_j\mid \Lowsig_i) - \epsilon. \tag{High effort, truthful v.s. Low effort, either truthful or untruthful} \\
    \mathbb{E}[\textsc{GPPM}(\tau(\Signal_i), \Signal_j)] - \mathbb{E}[\textsc{GPPM}(\strategy(\textsc{Null}), \Signal_j)] &\geq I(\Signal_i;\Signal_j) - \epsilon \tag{High effort, truthful v.s. No effort, either truthful or untruthful}. 
\end{align*}

Note that the mutual information $I(\Signal_i;\Signal_j\mid \Lowsig_i)$ and $I(\Signal_i;\Signal_j)$ are non-negative.
\end{restatable}

\begin{restatable}{proposition}{propGSPPM}\label{prop:GSPPM}
For the \textsc{Generative Synopsis Peer Prediction Mechanism (GSPPM)}, when \Cref{assum:llm-prior} holds with parameter $\epsilon'\geq 0$, for any agent $i$, given the peer agent $j$ exerting high effort and reporting truthfully,  any untruthful reporting strategy $\hat{\sigma}\neq \tau$ or effort $c_i\in\{0,c_l\}$ (no-effort or low-effort) implying signal $\textsc{Null}$ or $\Lowsig_i$ won't bring more than $\epsilon$ score increase. Specifically, we have
\begin{align*}
    \mathbb{E}[\textsc{GSPPM}(\tau(\Signal_i), \Signal_j)] - \mathbb{E}[\textsc{GSPPM}(\hat{\sigma}(\Signal_i), \Signal_j)] &> - \epsilon'. \tag{High effort, truthful v.s. High effort, untruthful}\\
    \mathbb{E}[\textsc{GSPPM}(\tau(\Signal_i), \Signal_j)] - \mathbb{E}[\textsc{GSPPM}(\strategy(\Signal^l_i), \Signal_j)] &\geq I(\Signal_i;\Signal_j\mid \Synopsis ,\Lowsig_i) - \epsilon'. \tag{High effort, truthful v.s. Low effort, either truthful or untruthful} \\
    \mathbb{E}[\textsc{GSPPM}(\tau(\Signal_i), \Signal_j)] - \mathbb{E}[\textsc{GSPPM}(\strategy(\textsc{Null}), \Signal_j)] &\geq I(\Signal_i;\Signal_j \mid \Synopsis) - \epsilon' \tag{High effort, truthful v.s. No effort, either truthful or untruthful}. 
\end{align*}

Note that the mutual information $I(\Signal_i;\Signal_j\mid \Synopsis, \Lowsig_i)$ and $I(\Signal_i;\Signal_j\mid \Synopsis)$ are non-negative.
\end{restatable}

To prove the above results, we analyze agent $i$'s expected scores obtained by each strategy, given that agent $j$ exerts high effort and reports truthfully. When agent $i$ exerts a high effort and reports truthfully, we provide a lower bound of her expected score based on the LLM-Prediction Assumption (\Cref{assum:llm-prior}), which is the negative entropy of agent $j$'s high-effort signal, conditioning on agent $i$'s high-effort signal. For other strategies, we provide an upper bound of agent $i$'s expected score based on the fact that the log scoring rule is proper. The expected score will be at most the the negative entropy of agent $j$'s high-effort signal, conditioning on agent $i$'s signal. Finally, because the high-effort signal is more informative than the low-effort signal, we prove that exerting high effort and reporting truthfully provides approximately the highest expected score. The gap between the negative conditional entropy can be interpreted as mutual information, as stated in the propositions. We defer the proof to \Cref{app:proofs}. 

We now discuss the potent properties of the GPPM and GSPPM. 
Intuitively, if a mechanism can lead to a positive gap between the performance scores, by linearly rescaling the performance scores as payments with parameters $\alpha$ and $\beta$, it can motivate effort and truth-telling as long as the gap between exerting high effort and exerting low (or zero) effort overcomes the gap between the cost of effort, and consequently, it is potent. 

Formally, for the approximately potent properties, we have the following propositions. Again, we defer the proof to \Cref{app:proofs}. 

\begin{restatable}{proposition}{propGPPMpotent}\label{prop:GPPM-potent}

For the \textsc{Generative Peer Prediction Mechanism (GPPM)}, when \Cref{assum:llm-prior} holds with parameter $\epsilon\geq 0$. When $I(X_i;X_j \mid \Lowsig_i)> 0$, GPPM is $\alpha\epsilon$-potent, where
\[\alpha = \max\left(\frac{c_h-c_l}{ I(\Signal_i;\Signal_j\mid \Lowsig_i)}, \frac{c_h}{I(\Signal_i;\Signal_j)}\right).\] When $\epsilon=0$, GPPM is potent. 

\end{restatable}

\begin{restatable}{proposition}{propGSPPMpotent}\label{prop:GSPPM-potent}

For the \textsc{Generative Synopsis Peer Prediction Mechanism (GSPPM)}, when \Cref{assum:llm-prior} holds with parameter $\epsilon'\geq 0$. When $I(X_i;X_j \mid \Synopsis,\Lowsig_i)> 0$, GSPPM is $\alpha\epsilon'$-potent, where
\[\alpha = \max\left(\frac{c_h-c_l}{ I(\Signal_i;\Signal_j\mid \Synopsis,\Lowsig_i)}, \frac{c_h}{I(\Signal_i;\Signal_j \mid \Synopsis)}\right).\] When $\epsilon'=0$, GSPPM is potent. 

\end{restatable}

These two propositions show that better LLM-prediction approximations will lead to better incentive properties. On the other hand, to achieve any desired incentive property, there's a corresponding threshold for the error of the LLM-prediction approximation.

Additionally, we compare GPPM and GSPPM. Let $\text{Gap}(h,l)$ denote the difference in the expected score of agent $i$ between exerting high effort and exerting low effort while both agents report truthfully. Let $\text{Gap}(h,\text{Null})$ be the analogous notation for the gap between exerting high effort and exerting no effort. Combining \cref{prop:GPPM} and \cref{prop:GSPPM} together, we have the following corollary.

\begin{restatable}{corollary}{coroGAP}\label{coro:gap_bound}
If the low-effort signals are synopsis-determined (\cref{def:synopsis-determined}) and synopsis-covering (\ref{def:synopsis-covering}), and $\epsilon'=\epsilon$, $\text{Gap}(h,l)$ has the same lower bound, $I(X_i;X_j \mid \Synopsis)-\epsilon$, in both GPPM and GSPPM. In contrast, $\text{Gap}(h,\text{Null})$ has a smaller lower bound $I(X_i;X_j \mid \Synopsis)-\epsilon<I(X_i;X_j)-\epsilon$ in GSPPM than in GPPM. 
\end{restatable}

\Cref{coro:gap_bound} suggests that compared with GPPM, GSPPM shrinks the gap between no-effort and low-effort, while preserving the gap between low-effort and high-effort. This property of GSPPM offers two advantages over GPPM: 1) in practice, it is harder for agents to ``cheat'' the mechanism by submitting a low-effort signal and getting a partial payoff, which consequently further incentivizes high-effort signals; 2) it reduces the noise caused by low-effort signals and produces more reliable scores, better differentiating between low-effort and high-effort.

Although this comparison concerns the lower bounds of the gaps, we will confirm this theoretical insight by empirically showing that GSPPM can distinguish low-effort and high-effort reports better. We demonstrate this by using GSPPM and GPPM to score human-written and LLM-generated reviews in \Cref{sec:result3}.

\section{Estimating the Posterior Prediction via LLMs}\label{sec:method}

Although our mechanisms, at this point, may appear to be straightforward generalizations from prior work, implementing them in practice with textual reports presents distinct challenges. In this section, we present the implementations of our mechanisms, which crucially involve estimating the underlying distribution $\Pr[\Signal_j = \report_j \mid \Signal_i = \report_i]$ via LLMs. We introduce two heuristic implementation methods, each leveraging the capabilities of the LLM in different ways and degrees.

\begin{itemize}
    \item The first implementation, denoted as \textsc{Token}, leverages the LLM by directly accessing its output layer to obtain the log probability feedback of the next token, which requires access to the neural network (NN) of the LLM. It has two variants:
    \begin{itemize}
        \item \textsc{Token-Raw}: We directly use the log probability to predict agents' raw reports.
        \item \textsc{Token-Preprocess}: We first use the LLM to preprocess agents' reports and use the log probability to predict the pre-processed reports. The goal is to standardize the language styles and extract essential information.
    \end{itemize}
    \item The second implementation, denoted as \textsc{Judgment}, uses the the LLM to first distill each report into a set of ``judgments'' and further apply the LLM chatbot to estimate the likelihood of each judgment with textual response. This implementation is particularly useful when it is hard to access the output layer of the LLM, since logprob feedback usually cannot be obtained from commercial LLM APIs, such as GPT-4 Chat Completion.
\end{itemize}

Note that both implementations can be used for \textit{zero-shot} estimation, meaning that estimating $\Pr[\Signal_j = \report_j \mid \Signal_i = \report_i]$ only requires $\report_i$ and $\report_j$ without any historical data. We alternatively offer a non-zero-shot implementation that estimates judgment distributions by clustering historical data. We primarily focus on the zero-shot implementations in the main body, while we defer the discussions of the clustering implementation to \Cref{appendix:alter}. The full implementations and results are available at \url{https://github.com/yx-lu/Eliciting-Informative-Text-Evaluations-with-Large-Language-Models}.

\begin{table}[!ht]
\centering\renewcommand{\arraystretch}{0.85}
\setlength{\tabcolsep}{4pt}
\begin{tabular}{cccccc}
\toprule
Implementation & Require network access & LLM usage\\ 
\midrule
\textsc{Token-Raw} & Yes & Generate log-probability\\
\midrule
\textsc{Token-Preprocess} & Yes & Preprocess, generate log-probability\\
\midrule
\textsc{Judgment} & No & Distill report into judgments, predict judgments\\
\bottomrule
\end{tabular}
\caption{A high-level comparison of three implementations.}
\label{table:compare}
\end{table} 

\subsection{\textsc{Token}: Implementation by LLM Token-Prediction}\label{sec:method1}

As discussed, the idea of \textsc{Token} is to use the log-probability (\textit{logprob}) feedback of an LLM to predict a given report $\report_j$. 
This is possible because LLMs are fundamentally designed to estimate the likelihood of a subsequent token in a sequence based on a distribution induced by the prompt. This inherent capability is deeply embedded in their pre-training datasets. Consequently, with access to an \textit{open-source} LLM, such as Llama-2~\cite{touvron2023llama} or ChatGLM~\cite{du2022glm,zeng2023glm-130b}, it becomes feasible to compel the LLM to generate a specific output and report the logprob for each output token. 

\begin{figure}[!ht]
\noindent\begin{minipage}[t]{0.9\textwidth}
\begin{tikzpicture}
\node[fill=gray!7,minimum width=\textwidth,minimum height=1.6cm,text width=0.9\textwidth] (box) {\textbf{Prompt $\psi(\report_i)$: }You are the second reviewer for a scientific paper. You are given a peer review from the other reviewer: \textbf{[Review $\report_i$]} Your task is to provide your own judgments of the paper based on the given materials. 
}; 
\end{tikzpicture}
\end{minipage}
\noindent\begin{minipage}[t]{0.9\textwidth}
\begin{tikzpicture}
\node[fill=blue!7,minimum width=\textwidth,minimum height=1.6cm,text width=0.9\textwidth] (box) {\textbf{Response:} [Predicted Review $\report_j$] 
 \\~\\ \textbf{Logprob = $\log \Pr_{\LLM(\psi)}[\Signal_j = \report_j \mid \Signal_i = \report_i]$}}; 
\end{tikzpicture}
\end{minipage}
    \caption{A simplified example of a prompt for \textsc{Token} in academic peer review scenario. In Appendix~\ref{appendix:prompt}, we present all exact prompts used in our experiment.}
    \label{fig:prompt-eg}
\end{figure}

Formally, we view the textual signal $\Signal = \{\Signal^{(k)}\}_{k \in [n]}$ as a sequence of $n$ tokens, where $\Signal^{(k)}$ denotes the $k$-th token within $\Signal$. Again, we use $\report^{(k)}$ to denote the value of the $k$-th token in a report. 

Given a prompt $\psi$ and the first $k-1$ tokens, an open-source LLM can provide the conditional distribution for the subsequent token via \textit{logprob} feedback, denoted as $\Pr_{\operatorname{LLM}(\psi)}[\Signal^{(k)} \mid \Signal^{(l)} = \report^{(l)} ~\forall l \in \{1,2,\ldots,k-1\} ]$. With Bayes' rule, we can write the probability of the occurrence of a given report $\report$ as:
\[
    \Pr_{\operatorname{LLM}(\psi)}[\Signal=\report] = \Pr_{\operatorname{LLM}(\psi)}\left[\Signal^{(1)} = \report^{(1)}\right] \prod_{k=2}^{n} \Pr_{\operatorname{LLM}(\psi)}\left[\Signal^{(k)} = \report^{(k)} ~\middle|~ \Signal^{(l)} = \report^{(l)} ~\forall l \in \{1,2,\ldots,k-1\}\right].
\]

\subsubsection{\textsc{Token-Raw}} 
The most straightforward idea is to integrate the raw report of agent $i$ into the prompt $\psi_{\text{token}}(\report_i)$ and compute $\Pr_{\operatorname{LLM}(\psi_{\text{token}}(\report_i))}[\Signal=\report_j]$ with the log probability output by the LLM. We view this probability as an approximation for $\Pr[X_j=\report_j \mid X_i=\report_i]$. Finally, agent $i$'s performance score is 
\[\textsc{Token-Raw}(\report_i,\report_j) =\log \Pr_{\operatorname{LLM}(\psi_{\text{token}}(\report_i))}[\Signal=\report_j].\]

\subsubsection{\textsc{Token-Preprocess}} 

In practice, agents' raw reports can vary significantly in language style, such as vocabulary usage, sentence structure, and grammatical errors. Additionally, reports may contain superficial information, such as a summary of the paper, in cases of peer review.

Note that even low-quality reviews can be well correlated with superficial information or language style. However, we aim to reward agents based on the quality of the reports' semantics, rather than the correlation on superficial information or language style. Such information may provide ``shortcuts'' that confound the LLM predictions and consequently lead to unintended rewards. Therefore, we filter out the shortcut information, including language style and superficial information, before applying LLMs to predict the responses. 

To address these issues, we propose a simple yet effective preprocessing technique that involves employing a uniform LLM (It may be the same as or different from the LLM generating logprob) to rephrase the text signal into a pre-set format. Note that the preprocessing step should be tailored for different tasks, considering the trade-off between retaining details and distilling the essential information. The details of our preprocessing prompt are presented in \Cref{app:pre-process} for reference.

Formally, the performance score of agent $i$ using the implementation $\textsc{Token-Preprocess}$ can be defined as:
\[\textsc{Token-Preprocess}(\report_i,\report_j) = \textsc{Token-Raw}(\textsc{Preprocess}(\report_i),\textsc{Preprocess}(\report_j)).\]

We will provide evidence comparing \textsc{Token-Raw} and \textsc{Token-Preprocess} and discuss this further in \Cref{sec:need-for-preprocess}. Generally, we find that \textsc{Token-Preprocess} is likely to effectively filter out such shortcut information and provide scores according to the semantic quality. Therefore, our main paper will mainly discuss the variant \textsc{Token-Preprocess}.

\subsection{\textsc{Judgment}: Implementation by LLM Judgment-Prediction}\label{sec:method2}

Accessing the output layer of an LLM can sometimes be impossible, especially when the LLM is not open-source. For example, the popular OpenAI GPT-4 chat completion API can only output sampled textual responses rather than the probability of tokens. We propose an alternative method that uses an LLM as a black box. Our idea is to first summarize each report as a set of judgments and then predict the peer's report by predicting the probability of each of its judgments and taking the product.

Formally, suppose the set of all possible judgments is $J = \{w_1, w_2, \ldots, w_m\}$. Suppose each signal $\signal$ (and report $\report$) is a subset of judgments, i.e., $\signal\subset J$. We assume that the event of whether each judgment belongs to a report is independent, thus, we have \[\Pr[\Report_j = \report_j \mid \Report_i = \report_i] = \prod_{w\in \report_j} \Pr\left[w \in \report_j \mid \Report_i = \report_i\right] \cdot \prod_{w\in J\backslash \report_j} \Pr\left[w \not\in \report_j \mid \Report_i = \report_i\right].\]

In the zero-shot setting, it is hard to access the full universe of judgments $J$, making it infeasible to estimate the probability $\Pr\left[w \not\in \report_j \mid \Report_i = \report_i\right]$. Therefore, we use $\prod_{w\in \Report_j} \Pr\left[w \in \Report_j \mid \Report_i = \report_i\right]$ as a heuristic predictor of $\Pr[\Report_j = \report_j \mid \Report_i = \report_i]$.

Furthermore, we discretize the prediction space for LLMs' responses. In particular, given a review $\report_i$ and a target judgment $w$, we ask the LLM to score how much $w$ contradicts or supports $\report_i$ with the score ranging from $-3$ (strong contradiction) to $3$ (strong support). We provide exact prompts used in our experiment in Appendix~\ref{appendix:prompt}. We view this score as the gain in log probability, i.e., $\log \Pr\left[w \in \Report_j \mid \Report_i = \report_i\right] - \log \Pr\left[w \in \Report_j\right]$.

By taking the sum over the judgments in $\report_j$, we obtain an estimate of the conditional log probability, which we use as the performance score, i.e.
\[\textsc{Judgment}(\report_i,\report_j) = \log \Pr[\Report_j = \report_j \mid \Report_i = \report_i] - \log \Pr[\Report_j = \report_j].\] 
Note that the performance score is not exactly the log of the conditional probability (the first term) as defined in \cref{def:GPPM}. However, subtracting the second term which is independent of agent $i$'s strategy will not disturb the incentive of agent $i$.

We emphasize that this method is not an accurate estimate of the goal probability $\Pr[\Signal_j = \report_j \mid \Signal_i = \report_i]$ in general. However, it is a plausible and feasible solution given a black-box LLM. As we will see in the experiments, this heuristic simplification can still capture key information within agents' reports. We note that the performance of our method may be improved with future versions of LLMs that can output more calibrated and accurate predictions.

\section{Evaluation}\label{sec:evaluation}

This section presents the methods we use to empirically evaluate the efficacy of our mechanisms. We are primarily interested in testing whether replacing the original high-quality reports with less informative low-quality reports leads to an expected performance score decrease with our mechanisms. We hypothesize that these low-quality reports can be viewed as reports by low-effort agents, and thus, a decrease in performance scores would indicate the empirical effectiveness of our mechanisms in eliciting high effort.

Toward this goal, we first introduce a general workflow and then propose several methods to create low-quality reports, including degrading the original reports and replacing the original reports with LLM-generated fictitious reports, which are regarded as lower quality than the original human expert reports. Note that these methods generating low-quality reports are not necessarily to be restricted in the effort model in \Cref{subsec:model} as we focus on testing the efficacy in a realistic scenario without relying on the theoretical assumptions. We finally introduce the statistical metrics used to measure the significance.

Specifically, let $z$ represent an item randomly selected from the dataset. We randomly draw two reports related to $z$, denoted as $\report_i$ and $\report_j$. Given the mechanism $M$, the computed score of agent $i$ is $s_{+} := M(\report_i, \report_j)$. We then replace $\report_i$ with a low-quality report $\report_i^l$ and recalculate the score $s_{-} := M(\report_i^l, \report_j)$. Our hypothesis is that the score will decrease in expectation. To test this, we apply a $t$-test to evaluate whether the decrease in score from \( s_{+} \) to \( s_{-} \) is statistically significant, thereby confirming the sensitivity of our methods to manipulations. 

Formally, we describe the implementation of the above workflow in \Cref{alg:evaluation}.

\begin{algorithm}[]
\caption{Evaluation Workflow}
\label{alg:evaluation}

\KwIn{A dataset with items and associated text reports. A mechanism $M$ computing the peer prediction scores. A low-quality report generation process $D$.}
\KwOut{Lists $\mathbf{s_{+}}$ and $\mathbf{s_{-}}$ and the statistics metrics.}

\BlankLine

\For{$k = 1$ \KwTo $K$}{
    Draw an item $z$ u.a.r. from all the items\;
    Draw two text reports $\tilde{x_i}$ and $\tilde{x_j}$ u.a.r. from all reports related to $z$ \;
    Compute $s_{+}^{(k)} := M(\tilde{x_i}, \tilde{x_j})$\;
    
    \BlankLine

    Replace the text report $\tilde{x_i}$ with $\tilde{x_i}^l$ according to process $D$\;
    Compute $s_{-}^{(k)} := M(\tilde{x_i}^l, \tilde{x_j})$\;
    
    \BlankLine
}

Add the computed values to lists $\mathbf{s_{+}} = \{s_{+}^{(k)}\}$, $\mathbf{s_{-}} = \{s_{-}^{(k)}\}$\;

Compute the statistics metrics based on lists $\mathbf{s_{+}}$ and $\mathbf{s_{-}}$\;
\end{algorithm}

\subsection{Reports Degradations} 
\label{subsec:manipulations}

We regard the original responses from human agents as truthful and high-effort reports and create three degradation methods, which obviously degrade the information within agents' reports and, consequently, should lead to lower scores. We use these degradations as a ``sanity check'', implying that any mechanisms that fail to penalize these degradations are unlikely to be useful in practice. 

\paragraph{Random Report Replacement:} We replace a report \( \tilde{x}_i \) with a new report \( \tilde{x}_{i'} \), which is randomly selected from a different randomly selected item \( z' \). This process is denoted as \( \tilde{x}_i^l = \tilde{x}_{i'} \). Note that $\tilde{x}_{i'}$ can also be viewed as a zero-effort signal as the information is likely to be irrelevant. For example, this method corresponds to the behaviors of malicious customers who upload irrelevant reviews in exchange for a restaurant's discount rewards.

\paragraph{Sentence-Level Degradation:} We degrade the original reports by deleting every other sentence. 

\paragraph{Judgment-Level Degradation:} We degrade the list of judgments by deleting every other judgment. This can only be performed with the \textsc{Token-Preprocess} and \textsc{Judgment} implementations, as the preprocessing step has already provided a well-structured list of the judgments. 

The sentence-level and judgment-level degradation methods only depend on the agent's report. Thus, they can be viewed as not only creating a low-effort signal but also untruthfully reporting a high-effort signal. The same experiments here (\Cref{sec:result1}) can also be used to test whether they can incentivize truth-telling.

\subsection{LLM-generated Reviews.} 
\label{subsec:LLM-generated}

Furthermore, we employ the LLMs to create synthetic text reports based on a given item, simulating the scenario of the creation of fictitious academic peer reviews. We conduct the experiment on the ICLR2020 OpenReview dataset\footnote{As in 2020, the generative AI is not as widely used as now, we assume all the academic peer reviews in the dataset are written by humans.}. Specifically, we provide the paper $z$ as input to both GPT-3.5 and GPT-4 \citep{achiam2023gpt,brown2020language}, requesting them to generate comprehensive reviews following \citet{liang2023large}'s method. The prompt for this task is presented in \Cref{app:prompt_review}. 

We compare three types of reviews: the human-written review, the GPT-4-generated review, and the GPT-3.5-generated review. Given that GPT-4 is commonly considered a stronger AI than GPT-3.5, and both models are considered worse than human expert reviewers, we hypothesize that we can utilize the LLM-generated reviews to simulate low-quality human reviews, thereby simulating three levels of quality.

It is worth noting that, unlike the degradations discussed in \cref{subsec:manipulations}, the LLM-generated review does not only depend on the agent's signal $\signal_i$ since the reviewed paper is input into the LLMs to create reviews. It can not be regarded as untruthfully reporting the original signal. Thus, our focus remains on assessing the effectiveness of differentiating various quality levels of reports across different mechanisms.

\subsection{Statistics Metrics}
\label{subsec:measures}

The principal criterion for our method's effectiveness is a statistically significant decrease in scores following the degradations. To empirically validate this criterion, we conduct statistical tests on the samples of scores computed with the original reports (\( s_{+} \)) and the degraded/LLM-generated reports (\( s_{-} \)). Specifically, We employ the t-test, which is able to identify whether there is a significant decrease in scores from \( s_{+} \) to \( s_{-} \). Our empirical distributions (Section~\ref{sec:empirical-distribution}) show that the measured differences ($s_{+}-s_{-}$) are approximately normally distributed, confirming our data is suitable for the t-test. Below, we introduce the detailed implementation of the $t$-test.

\paragraph{Paired Difference t-test} 

For each pair of scores $(s_+, s_-)$, we use the paired difference t-test to verify if the mean of their difference is greater than 0. 
Let $d_k:=s_{+}^{(k)}-s_{-}^{(k)}$, and $\bar{d}$ denote the mean of the differences $d_k$. Formally, the hypotheses for the Paired Difference t-test are:

\begin{itemize}
    \item Null Hypothesis (\(H_0\)): The mean difference $\bar{d}$ is equal to 0.
    \item Alternative Hypothesis (\(H_1\)): The mean difference $\bar{d}$ is greater than 0.
\end{itemize}

The test statistic for the Paired Difference t-test is calculated as follows:
\[
t = \frac{\bar{d}}{\sigma_d / \sqrt{K}}
\]
where \(\sigma_d\) is the standard deviation of the differences $\{d_t\}$, and \(K\) is the number of pairs. Furthermore, we denote $SE(\bar{d}) = \frac{\sigma_d}{\sqrt{K}}$, which is the standard error of the mean difference.

The p-value is obtained by $p = \Pr[X > t]$, where \( X \) is a random variable following the t-distribution with the calculated degrees of freedom \( K-1 \), and \( t \) is the calculated t-statistic above. Typically, a p-value threshold of 0.05 is used, where a p-value lower than 0.05 leads to the rejection of the null hypothesis, indicating that there's a statistically significant difference. 

\section{Experiments}

This section presents the setup of our experiments and the empirical results. 

\subsection{Experiment Setup}

\subsubsection{Datasets}

We use two datasets for our experiment. 

\begin{description}
    \item[Yelp Online Review Data (Yelp)] is publicly available online review data from Yelp. The restaurants featured on Yelp are interpreted as the ``items'' in our model, and customer reviews for each restaurant are interpreted as the ``text reports'' for our model. We construct our dataset by randomly sampling 1000 items (restaurants) from the entire dataset. For these 1000 items, we have 198,444 text reports (customer reviews) in total, i.e., averaged about 200 reviews per restaurant.
    \item[ICLR Peer Review Data (ICLR)] includes peer review data from the International Conference on Learning Representations (ICLR) 2020, accessed via the OpenReview API.\footnote{The reason for using the ICLR 2020 dataset is to exclude the chance that reviewers use LLMs to generate their reports, as we discussed in \Cref{sec:evaluation}, in 2020, AI-generated reviews were rare.} In this dataset, the ``items'' are the submitted papers, and the ``text reports'' are the peer reviews they received. Given the typically longer and more informative nature of ICLR reviews compared to Yelp reviews, we choose a smaller sample size for the ICLR dataset to manage computational demands efficiently. Thus, we construct our dataset by randomly sampling 300 items (papers) from the entire dataset. For these 300 items, we have 911 text reports (peer reviews) in total, i.e., averaged about 3 reviews per paper. 
\end{description}

The Yelp dataset represents a crowdsourcing setting where reviews are completed by the public. On the other hand, the ICLR dataset can be viewed as an example of expert sourcing, involving reviews provided by experts in a highly specialized field. 

\subsubsection{LLMs}

\begin{description}
\item[GPT-4 / GPT-3.5] We employ the \textsf{gpt-4-1106-preview} \citep{achiam2023gpt} model for preprocessing the reports on the ICLR dataset and the \textsf{gpt-3.5-turbo-1106} \citep{brown2020language} model for preprocessing the reports on the Yelp dataset.\footnote{This selection is based on the nature of the texts in each dataset: Yelp reviews tend to be shorter and less complex, thus not requiring the advanced capabilities of a more powerful language model. In contrast, ICLR reviews are more intricate, justifying the use of the higher-capacity \textsf{gpt-4-1106-preview} model for effective rephrasing.} For the $\textsc{Judgment}$ implementation, we use \textsf{gpt-4-1106-preview} to predict judgments.
\item[Llama-2] For the $\textsc{Token}$ implementations, including both \textsc{Token-Preprocess} and \textsc{Token-Raw} (As discussed in \Cref{sec:method1}, we defer the results for \textsc{Token-Raw} to \Cref{sec:need-for-preprocess}), we use the \textsf{llama-2-70b-chat} \citep{touvron2023llama} model with 4-bit quantization to calculate log probabilities for token prediction. The open-source nature of Llama-2 allows for local execution and access to the log probability for each token in a text report.
\end{description} 

\begin{table}[htbp]
\centering\renewcommand{\arraystretch}{0.85}
\setlength{\tabcolsep}{10pt}
\begin{tabular}{ccc}
\toprule
\textbf{Task} & \textbf{LLM Usage} & \textbf{Notes}\\ 
\cmidrule{1-3}
\multirow{2}{*}{Preprocessing}  & ICLR: \textsf{gpt-4-1106-preview}  & \multirow{2}{*}{Discussed in footnote\footnotemark[6]}\\
                       & Yelp: \textsf{gpt-3.5-1106-preview} & \\
\cmidrule{1-3}
\textsc{Token} implementations & \textsf{llama-2-70b-chat} & Open source LLM\\
\cmidrule{1-3}
\textsc{Judgment} implementation & \textsf{gpt-4-1106-preview} & /\\
\cmidrule{1-3}
\multirow{2}{*}{Generating LLM reviews} & \textsf{gpt-3.5-1106-preview} & \multirow{2}{*}{Discussed in \Cref{subsec:LLM-generated}}\\
                       & \textsf{gpt-4-1106-preview}  \\
\bottomrule
\end{tabular}
\caption{Table illustrating the use of different LLM models in various tasks}
\label{tab:tasks_usage}
\end{table}

We provide all prompts used, along with example inputs and outputs, in \Cref{appendix:prompt}.

\subsubsection{Mechanisms}

\begin{description}
\item[GPPM.] We test the GPPM on both the Yelp and the ICLR datasets with \textsc{Token-Preprocess} and \textsc{Judgment} implementations. We perform the report degradation evaluation (\Cref{subsec:manipulations}) and LLM-generated-review evaluation (\Cref{subsec:LLM-generated}).

\item[GSPPM.] We test the GSPPM on the ICLR dataset. This is because the Yelp dataset we use lacks detailed features of the items (restaurants) and thus we do not have access to a suitable synopsis of each item. For the ICLR dataset, we consider the abstract of a paper as the commonly-known synopsis. Conditional on this, we apply the LSR to calculate the score in the GSPPM. 

We only apply the \textsc{Token-Preprocess} implementation for two reasons: First, when we have access to a neural network of an LLM, \textsc{Token-Preprocess} implementation is more time- and cost-efficient than \textsc{Judgment}, given there is no need for API calls. Second, the results in GPPM indicate that \textsc{Token-Preprocess} generally works better than \textsc{Judgment} on the ICLR dataset.

We are primarily interested in comparing its performance differentiating high-quality and low-quality reports. Thus, we only perform the LLM-generated-review evaluation (\Cref{subsec:LLM-generated}) to test the GSPPM. 

\item[Baseline.] Additionally, we present a baseline mechanism (\Cref{def:miller}) that uses only the numerical ratings from reports. This approach rebuilds the joint distribution of two ratings\footnote{There are 4 possible ratings in ICLR dataset and 5 possible ratings in Yelp dataset.} based on historical data and assigns a log scoring rule to agent $i$ based on her rating compared to agent $j$'s rating. Since it is not clear how to degrade numerical scores to the same degree as sentence/judgment-level degradations, this baseline is only applicable to the experiment of Random Report Replacement.

\end{description}

\subsection{Result 1: GPPM Effectively Penalizes Report Degradations.}\label{sec:result1}

We now delve into our main results. We first describe the results of evaluating the GPPM with three report degradations defined in \Cref{subsec:manipulations}. We apply sample sizes of $K=500$ and $K=1000$ for the experiments on the ICLR and the Yelp datasets respectively. We use both \textsc{Token-Preprocess} and \textsc{Judgment} to compute the estimated conditional probability. We visualize the p-values in \Cref{fig:bar-chart-token} and defer the comprehensive statistics metrics to \Cref{appendix:additional} (Table~\ref{table:random-replace}, \ref{table:sentence-degrade} and \ref{table:judgment-degrade}).

\begin{figure}[ht]
\centering
\subfigure[Results on ICLR Dataset]{\centering
\includegraphics[scale=0.39]{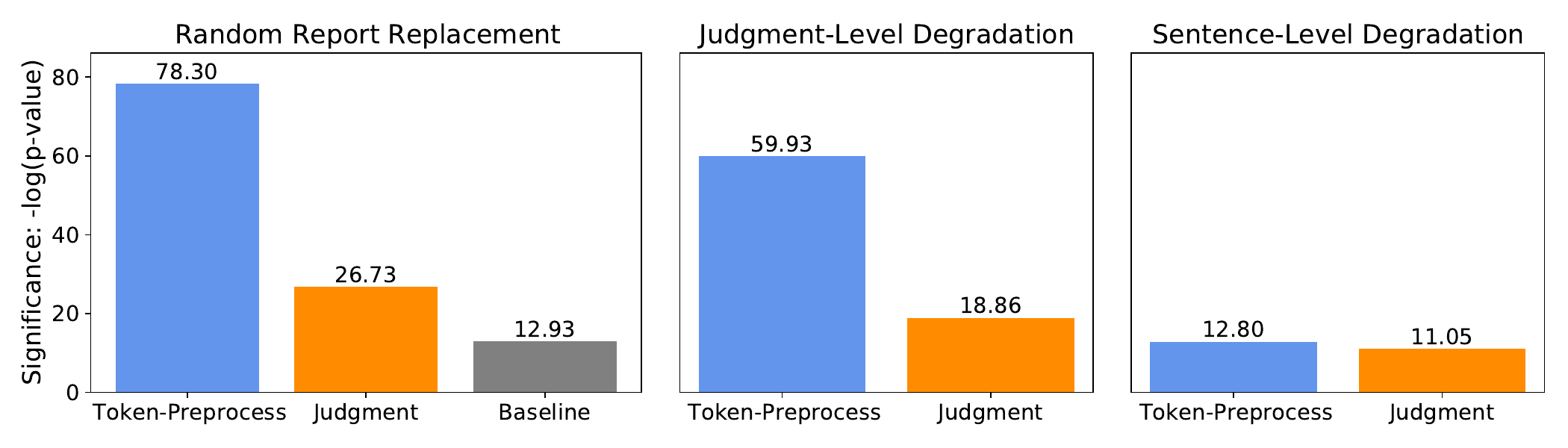}}

\subfigure[Results on Yelp Dataset]{\centering
\includegraphics[scale=0.39]{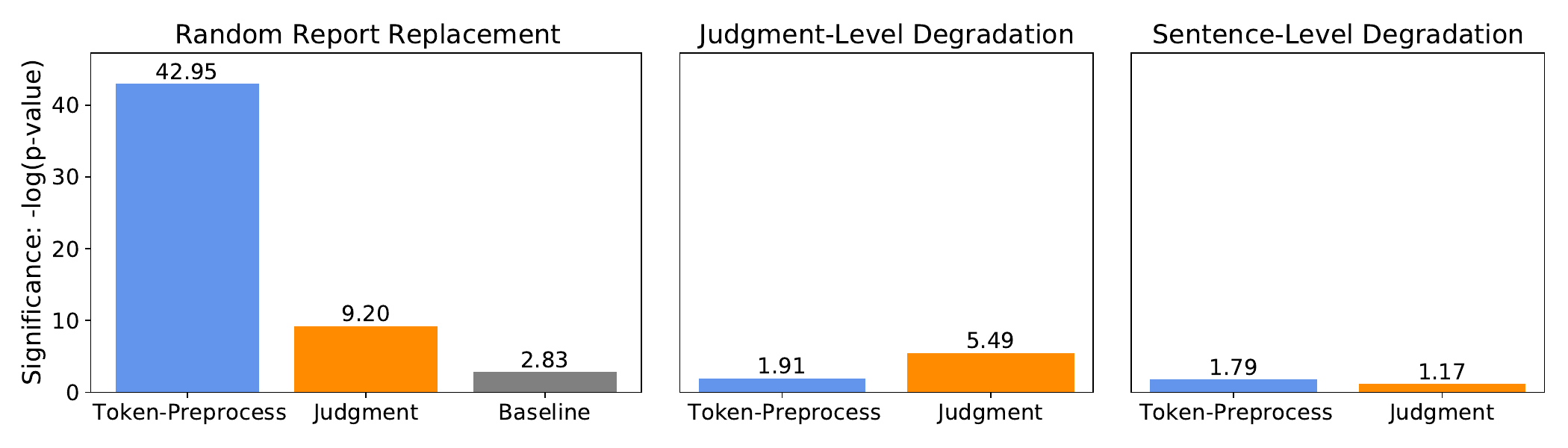}}
\caption{Report Degradation Evaluation Result for GPPM with \textsc{Token-preprocess} and \textsc{Judgment} Implementations, as well as the Baseline Mechanism: \textbf{Significance ($\bm{-\log_{10}}$(p-value))} of the expected score difference $\bar{d}>0$, \textbf{higher is better}. Typically, a significance score $-\log_{10}(\text{p-value})> 1.3$ (equivalent to p-value $< 0.05$) is regarded as significant difference. }
\label{fig:bar-chart-token}
\end{figure}

\textbf{GPPM significantly outperforms the baseline.}
We observe a positive $\bar{d}$ in all experiments: for three degradation methods, two datasets, and two implementations. Moreover, in the case of ``random report replacement'', although all tested mechanisms, including the baseline, exhibit a significance score ($-\log_{10}(\text{p-value})$) well above the threshold of $1.30$ (equivalent to p-values $< 0.05$), the significance score associated with the GPPM are significantly higher compared with the baseline. This observation matches our intuition that there exists a substantially larger amount of information within agents' textual responses, and our GPPM --- with either implementation --- can successfully extract it.

\textbf{\textsc{Token} outperforms \textsc{Judgment} on the ICLR dataset.}
Compared with the \textsc{Judgment} implementation, we observe a higher significance score (equivalent to lower p-values) of \textsc{Token-Preprocess} implementations on the ICLR dataset. However, \textsc{Token-Preprocess} does not perform well on the ``judgment-level degradation'' test conducted on the Yelp dataset. 

Additionally, one may be interested in the performance of \textsc{Judgment} with Llama-2. As GPT-4 has better inference capacity than Llama-2, the performance of \textsc{Judgment} with Llama-2 is worse than \textsc{Judgment} with GPT-4. We provide detailed results in \Cref{appendix:additional}. 
Hence, the \textsc{Token-Preprocess} implementation can be considered superior to \textsc{Judgment}, as it consistently outperforms \textsc{Judgment} when both are applied with the same LLM model. However, \textsc{Judgment} is still valuable when there is no access to the LLM's log probability feedback.

\subsection{Result 2: Both GPPM and GSPPM Differentiate Three Quality Levels --- human, GPT-4, and GPT-3.5}\label{sec:result2}

We now show the results of testing the GPPM and GSPPM with LLM-generated reviews introduced in Section~\ref{subsec:LLM-generated}. We focus on the ICLR dataset and use the \textsc{Token-Preprocess} implementation, given that the previous section shows that it is more efficient and performs better than \textsc{Judgment} on the ICLR dataset. We apply a sample size of $K=500$. We visualize the p-values in \Cref{fig:bar-chart-LLM-review-overall} and provide the statistics metrics in Table~\ref{table:llm-review}.

\begin{table}[htbp]
\small\centering\renewcommand{\arraystretch}{0.85}
\begin{tabular}{cccccccc}
\toprule
Mechanism & Review-Generating LLM & $\bar{d}$ & $\sigma(d)$ & $SE(\bar{d})$ & t-statistic & p-value & $-\log_{10}(\text{p-value})$\\ 
\cmidrule{1-8}
\multirow{2.5}{*}{GPPM}  & GPT-3.5 & 10.020 & 19.290 & 0.863 & 11.603 & 5.0e-28 & 27.297\\
\cmidrule{2-8}
                       & GPT-4 & 2.904 & 20.489 & 0.916 & 3.166 & 8.2e-04 & 3.086\\
\cmidrule{1-8}
\multirow{2.5}{*}{GSPPM} & GPT-3.5 & 9.197 & 14.173 & 0.634 & 14.495 & 2.7e-40 & 39.562\\
\cmidrule{2-8}
                       & GPT-4 & 5.357 & 14.716 & 0.658 & 8.131 & 1.7e-15 & 14.770\\
\bottomrule
\end{tabular}
\caption{Statistics Metrics of LLM-Generated Review Evaluation for GPPM/GSPPM implemented with \textsc{Token-Preprocess}. $\bar{d}$ represents the mean of the score differences, $\sigma(d)$ represents the standard deviation of the score differences, and $SE(\Bar{d})={\sigma(d)}/{\sqrt{K}}$ represents the standard error of the mean difference.}
\label{table:llm-review}
\end{table}

\begin{figure}[ht]
\centering
\includegraphics[scale=0.39]{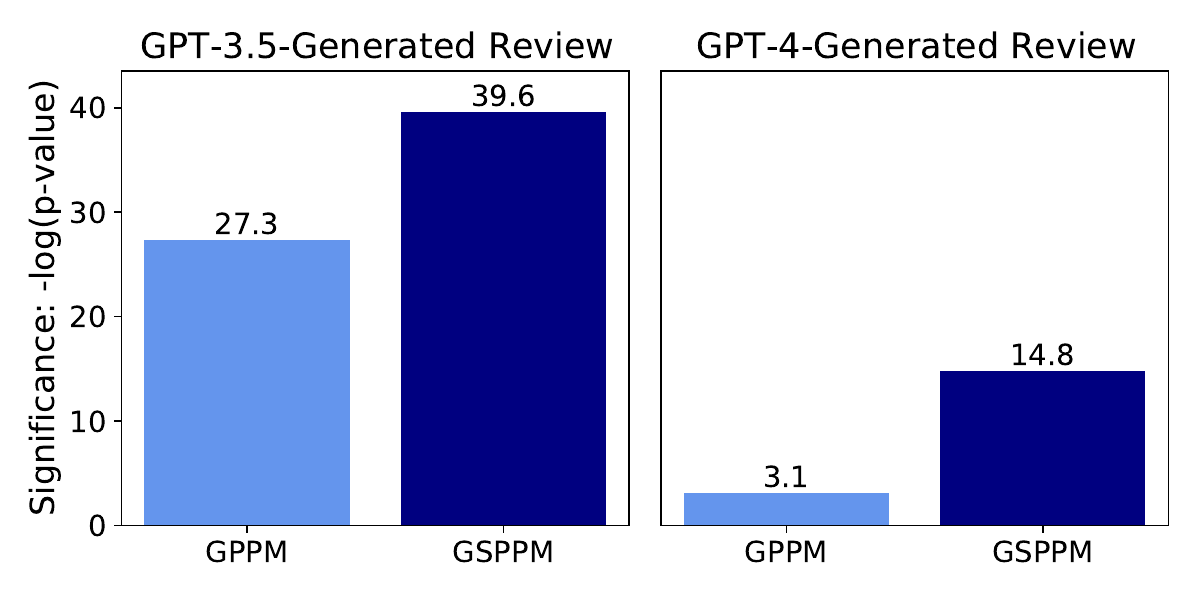}
\caption{LLM-generated Review Evaluation Result for GPPM and GSPPM with \textsc{Token-preprocess} Implementation: \textbf{Significance ($\bm{-\log_{10}}$(p-value))} of the expected score difference $\bar{d}>0$, \textbf{higher is better}. Typically, a significance score $-\log_{10}(\text{p-value})> 1.30$ (equivalent to p-value $< 0.05$) is regarded as significant difference. }
\label{fig:bar-chart-LLM-review-overall}
\end{figure}

In \Cref{table:llm-review}, we observe that, with both mechanisms and the \textsc{Token-Preprocess} implementation, replacing human-written reviews with either GPT-3.5 or GPT-4-generated reviews leads to a statistically significant decrease in scores, as the significance score are all more than the threshold of $1.30$ indicating p-values are all less than the threshold of $0.05$. 

Comparing the first row and the second row of Table~\ref{table:llm-review}, the scores are on the same scale as they are computed by the same mechanism. Thus, we observe that while the standard deviation  $\sigma(d)$ remains similar, replacing the human-written reviews with GPT-3.5-generated reviews leads to a greater reduction in the expected score, $\bar{d}$, than replacing with GPT-4-generated signals. Similarly, comparing the third row and the fourth row, we have the same observation.

Therefore, with both GPPM and GSPPM, we can observe three score levels, from high to low: human-written, GPT-4-generated, and GPT-3.5-generated. This observation suggests that both our mechanisms can effectively differentiate among these three quality levels. Furthermore, as hypothesized in \Cref{subsec:LLM-generated}, GPT-generated reviews can be viewed as low-effort responses. As GPT-4 generates higher-quality reviews than GPT-3.5, we can thus infer three effort levels: high (human-written), medium (GPT-4-generated), and low (GPT-3.5-generated). Therefore, our results also show the effectiveness of our mechanisms in differentiating various effort levels.

\subsection{Result 3: GSPPM Penalizes LLM-Generated Peer Peview More than GPPM.}\label{sec:result3} We compare the efficacy of GSPPM and GPPM in differentiating between human-written and LLM-generated reviews. As shown in \Cref{coro:gap_bound}, we expect that the GSPPM can do a better job than the GPPM. 

Note that the performance scores are no longer on the same scale. Therefore, we focus on the significance scores in \Cref{fig:bar-chart-LLM-review-overall}. We observe that the GSPPM has higher significance scores (lower p-values) in differentiating both GPT-3.5 and GPT-4-generated reviews, indicating its better performance at penalizing the LLM-generated reviews. Furthermore, note that the GPPM merely obtains a p-value of 8.2e-04 for the GPT-4-generated reviews, which is much larger than the GSPPM's p-value under the same condition, indicating a much lower significance.

Therefore, GSPPM has more significant score gaps among these three quality levels--human, GPT-4, and GPT-3.5. This is because much of the superficial information in the LLM-generated reviews is already contained in the synopsis (abstract). Therefore, GSPPM mitigates the ability of the LLMs to obtain a high score. Consequently, the more informative signals that require high effort to access but are necessary for a high-quality report tend to have a higher impact on the score computed by the GSPPM. As mentioned above, we take this as evidence that GSPPM can likely better differentiate high versus low-effort signals and thus better elicit high-effort reports than GPPM in the peer review scenario.

\subsection{Empirical Distribution of the Performance Score Change}\label{sec:empirical-distribution}

So far, we have been focused on the expected performance scores. However, in practice, the distribution of the performance matters. For example, risk-averse agents may be concerned with the frequency of receiving negative scores, and the variance in performance scores often reflects the fairness of the mechanism.

Here, we use kernel density estimation (KDE)\footnote{We use \textsf{seaborn.kdeplot} (seaborn.pydata.org/generated/seaborn.kdeplot.html) with default parameters to plot the KDEs. } to visualize the empirical distribution of $d=s_{+}-s_{-}$, the change of the performance score after applying the degradation. We observe that the score change follows a bell-shaped distribution. Furthermore, although the score changes are predominantly positive in all the cases, the probability sometimes approaches half, especially in the Yelp dataset. We hypothesize that this is because Yelp reviews tend to be shorter, more diverse, and less standardized, which greatly decreases the quality of the LLM predictions.

We emphasize that a positive score change in expectation is sufficient to incentivize high-quality reports. However, future research may seek to use the performance score as a metric for assessing data quality, in which case it is crucial to minimize the probability of getting a negative score change. We note that fine-tuning an existing LLM (which is not zero-shot) and the development of more advanced language models may contribute to potential improvements.

\begin{figure}[ht]
\centering
\subfigure[ICLR+\textsc{Token-Preprocess}]{\centering
\includegraphics[width=0.238\linewidth]{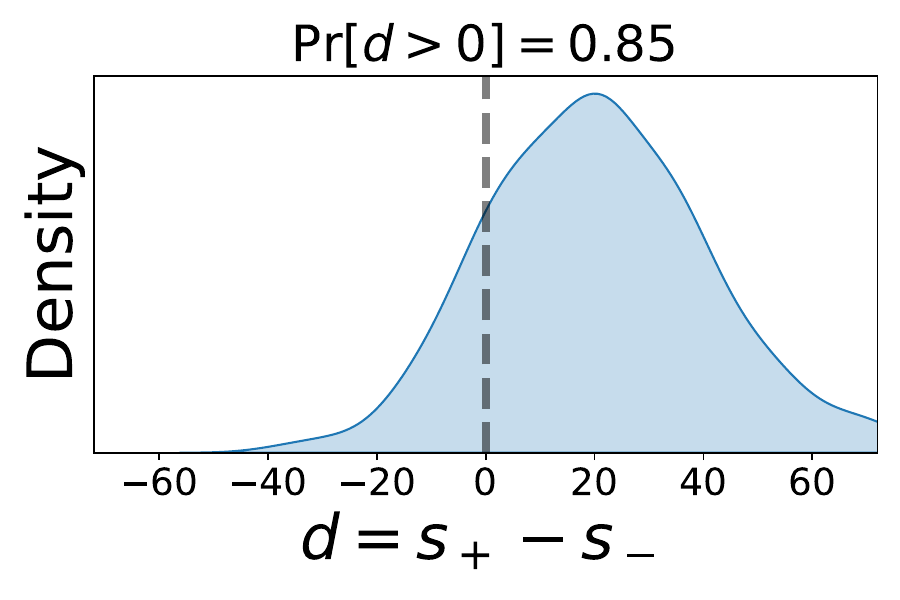}
}
\subfigure[ICLR + \textsc{Judgment}]{\centering
\includegraphics[width=0.238\linewidth]{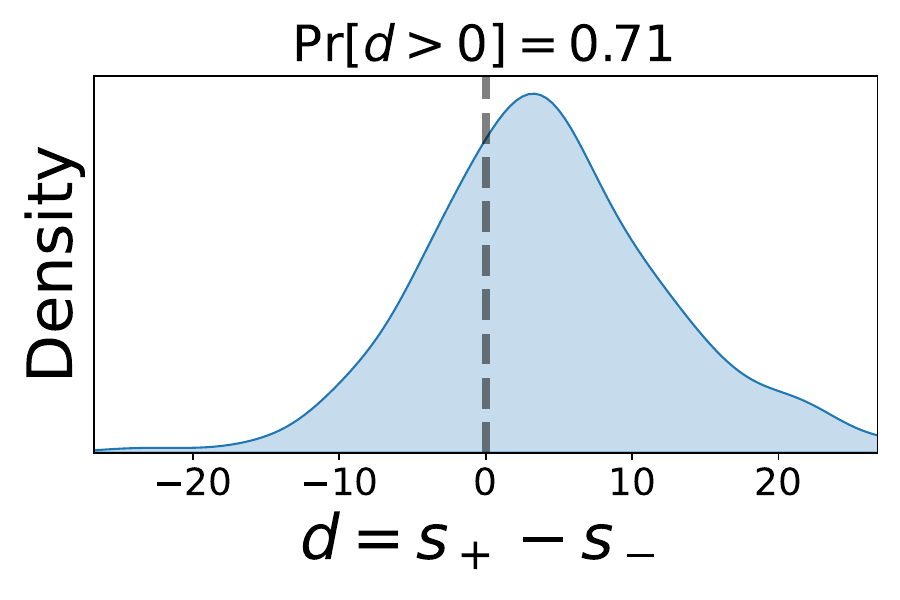}
}
\subfigure[Yelp + \textsc{Token-Preprocess}]{\centering
\includegraphics[width=0.238\linewidth]{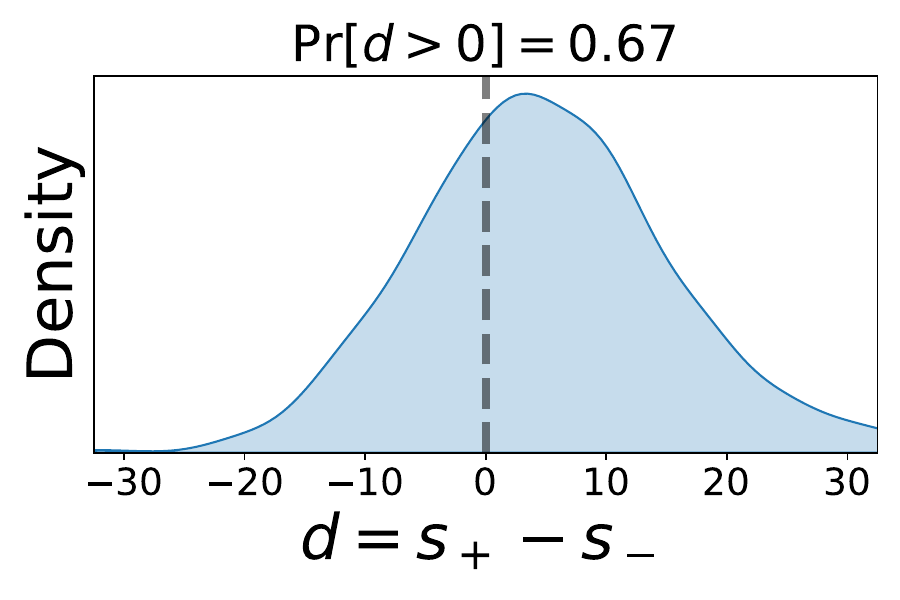}
}
\subfigure[Yelp + \textsc{Judgment}]{\centering
\includegraphics[width=0.238\linewidth]{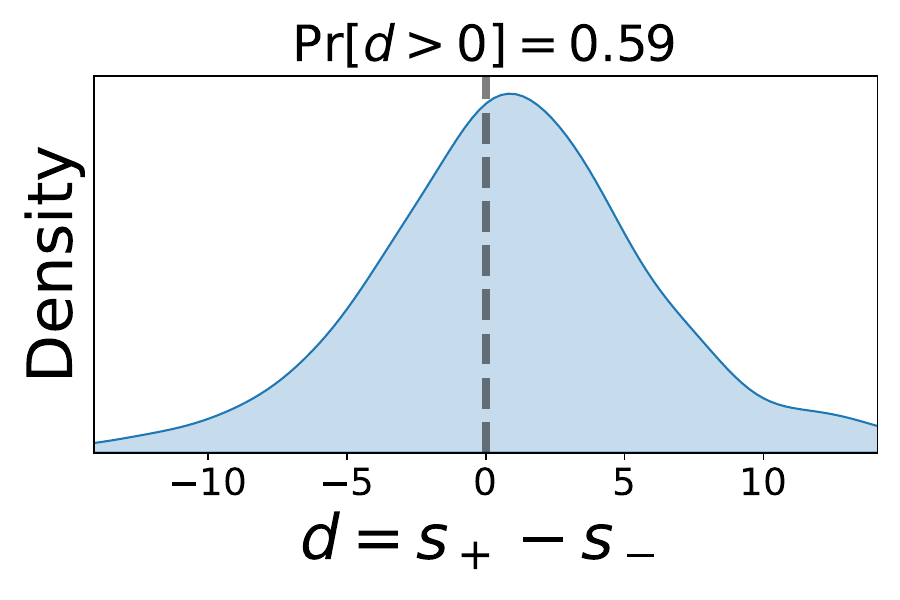}
}
\caption{Empirical Distribution of the Performance Score Change of Random Signal Replacement Evaluation. }
\label{fig:random-replace}
\end{figure}

\begin{figure}[ht]
\centering
\subfigure[ICLR+\textsc{Token-Preprocess}]{\centering
\includegraphics[width=0.238\linewidth]{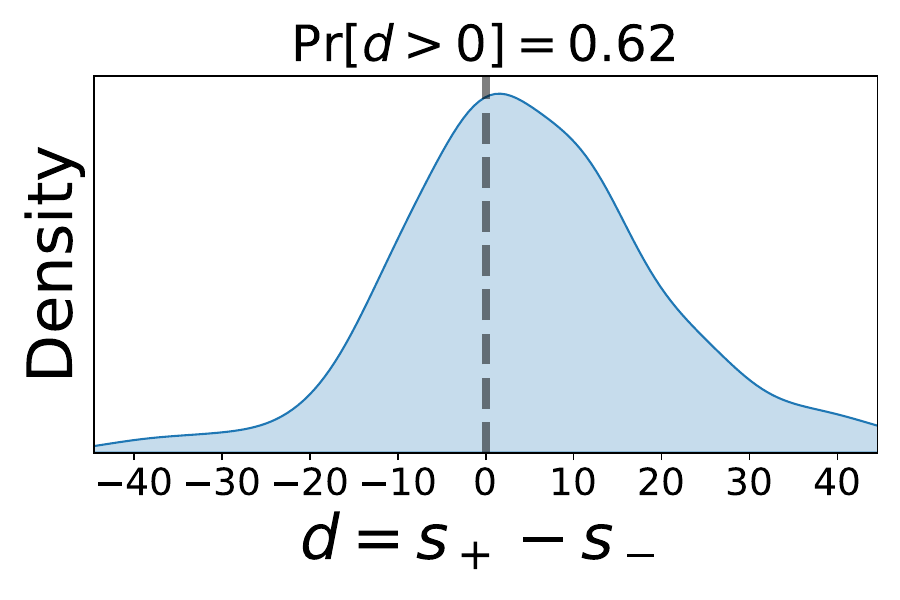}
}
\subfigure[ICLR + \textsc{Judgment}]{\centering
\includegraphics[width=0.238\linewidth]{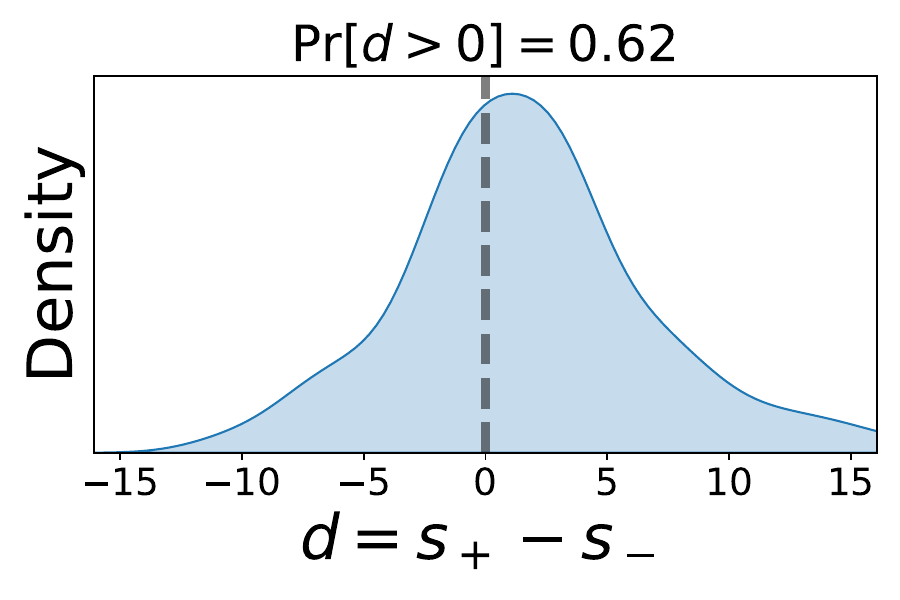}
}
\subfigure[Yelp + \textsc{Token-Preprocess}]{\centering
\includegraphics[width=0.238\linewidth]{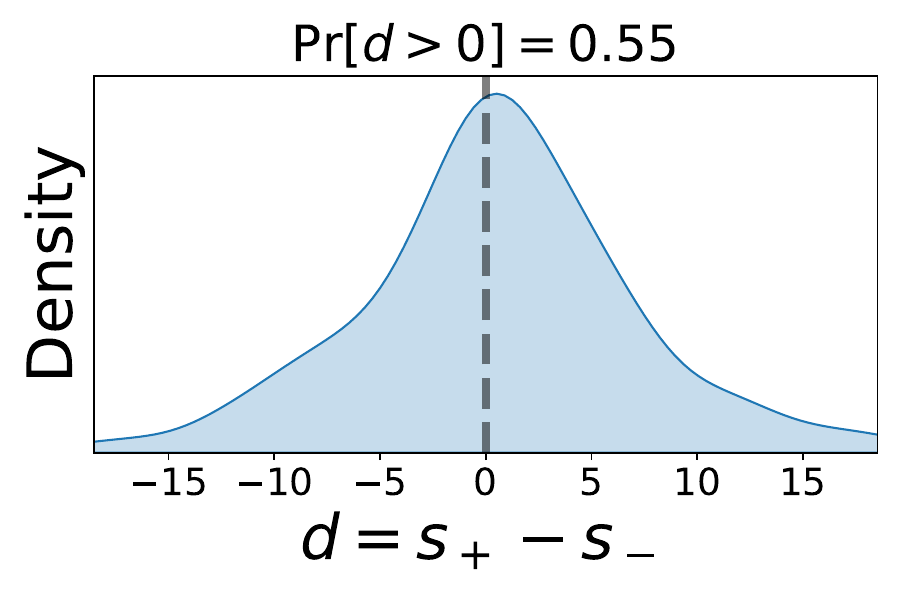}
}
\subfigure[Yelp + \textsc{Judgment}]{\centering
\includegraphics[width=0.238\linewidth]{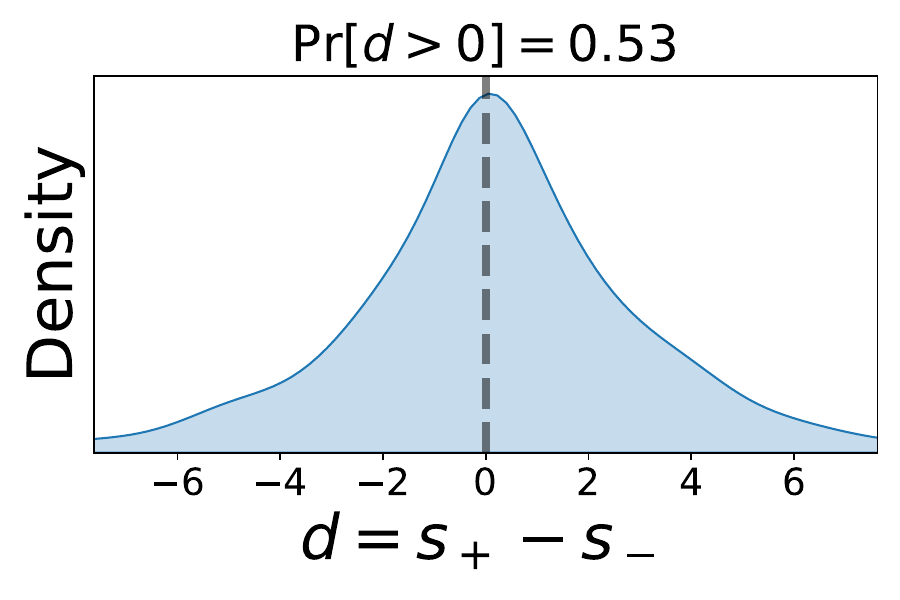}
}
\caption{Empirical Distribution of the Performance Score Change of Sentence-Level Degradation Evaluation. }
\label{fig:sentence-degrade}
\end{figure}

\begin{figure}[ht]
\centering
\subfigure[ICLR+\textsc{Token-Preprocess}]{\centering
\includegraphics[width=0.238\linewidth]{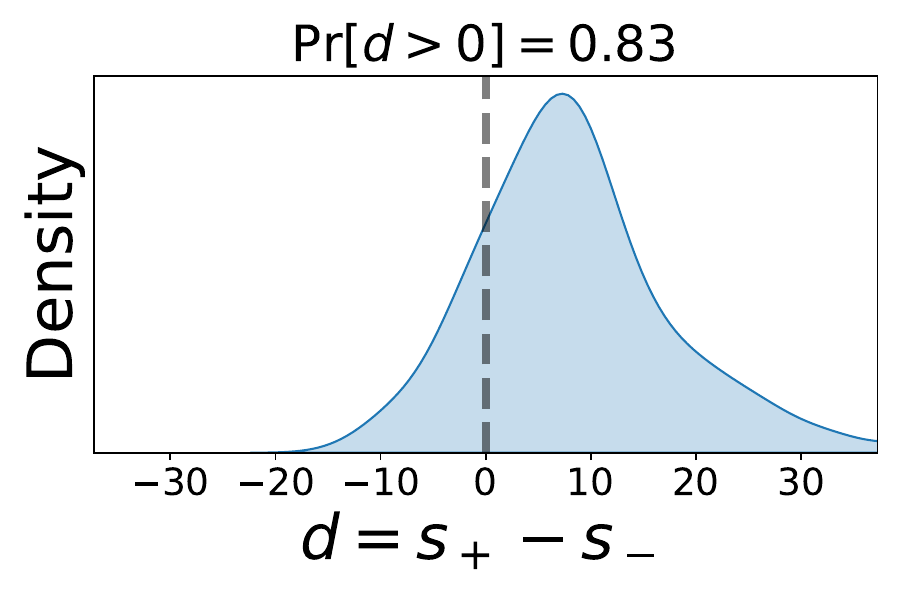}
}
\subfigure[ICLR + \textsc{Judgment}]{\centering
\includegraphics[width=0.238\linewidth]{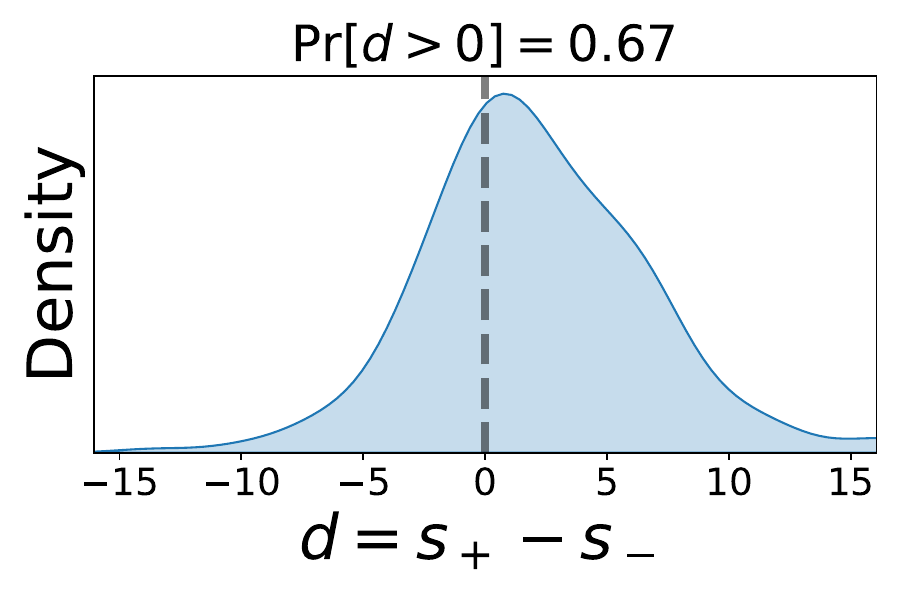}
}
\subfigure[Yelp + \textsc{Token-Preprocess}]{\centering
\includegraphics[width=0.238\linewidth]{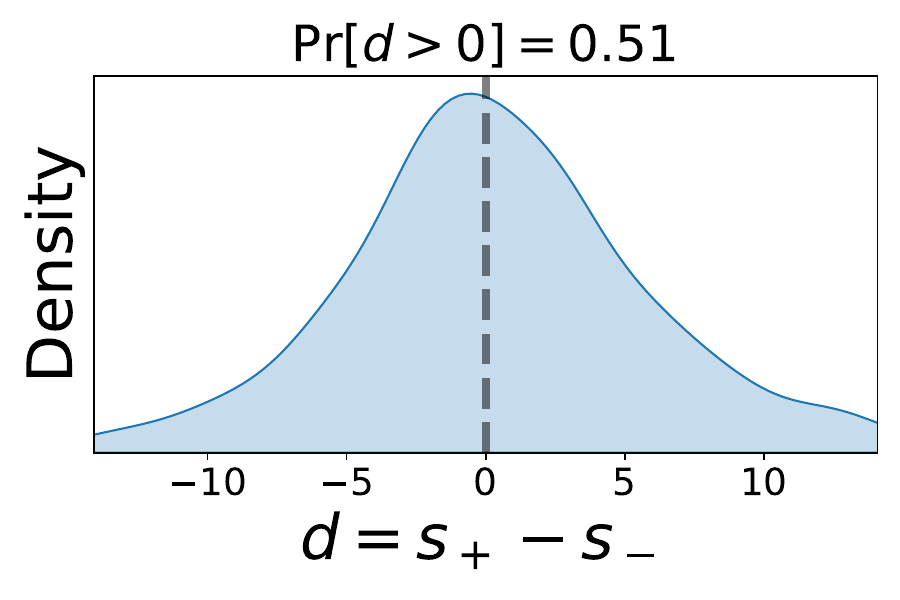}
}
\subfigure[Yelp + \textsc{Judgment}]{\centering
\includegraphics[width=0.238\linewidth]{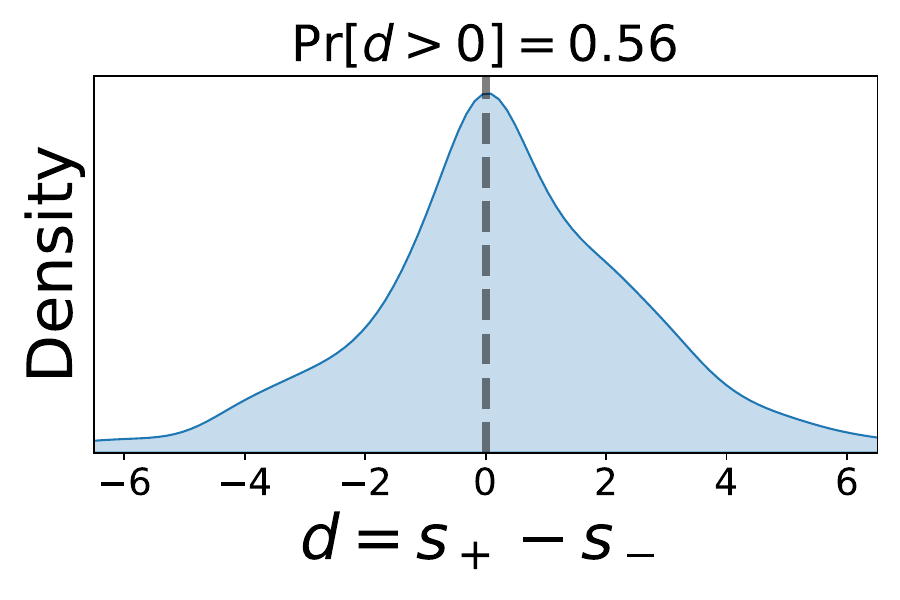}
}
\caption{Empirical Distribution of the Performance Score Change of Judgment-Level Degradation Evaluation. }
\label{fig:judgment-degrade}
\end{figure}

\begin{figure}[ht]
\centering
\subfigure[\textsc{GPPM} with\newline\hspace*{8pt}GPT-3.5-generated review]{\centering
\includegraphics[width=0.238\linewidth]{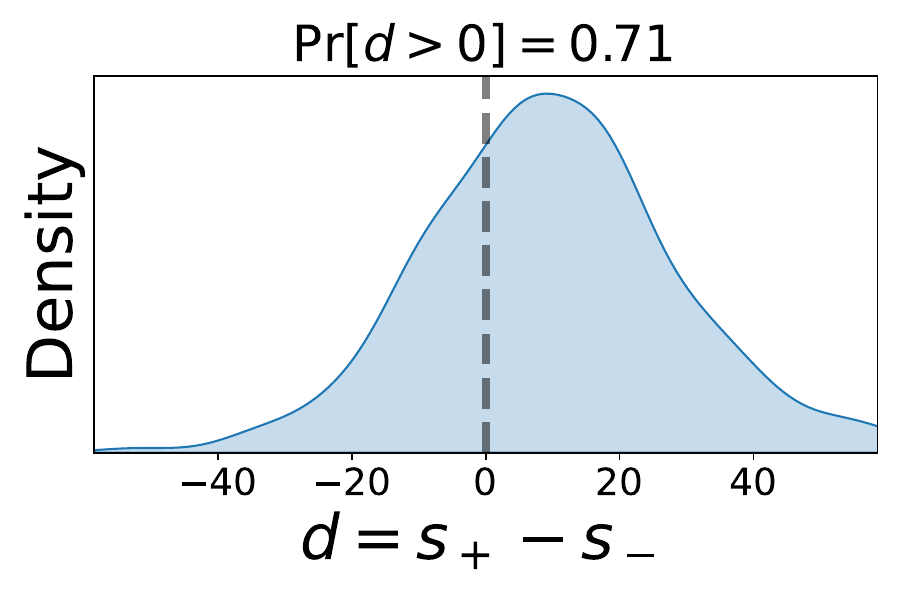}
}
\subfigure[\textsc{GPPM} with\newline\hspace*{8pt}GPT-4-generated review]{\centering
\includegraphics[width=0.238\linewidth]{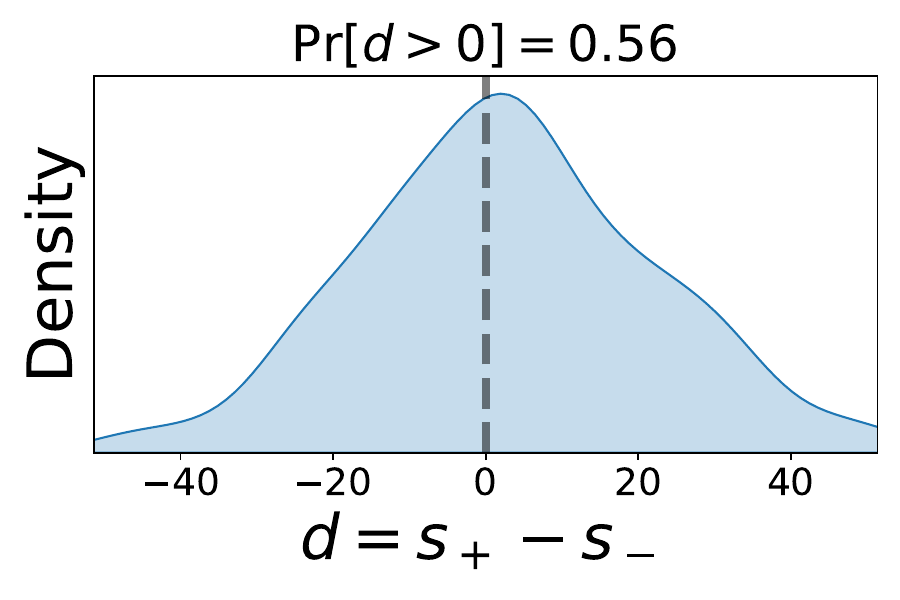}
}
\subfigure[\textsc{GSPPM} with\newline\hspace*{8pt}GPT-3.5-generated review]{\centering
\includegraphics[width=0.238\linewidth]{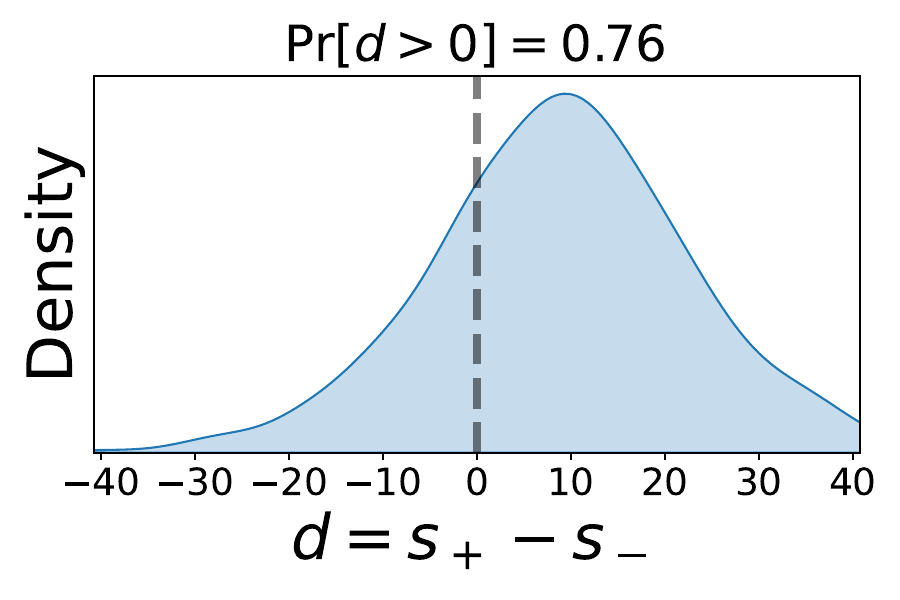}
}
\subfigure[\textsc{GSPPM} with\newline\hspace*{8pt}GPT-4-generated review]{\centering
\includegraphics[width=0.238\linewidth]{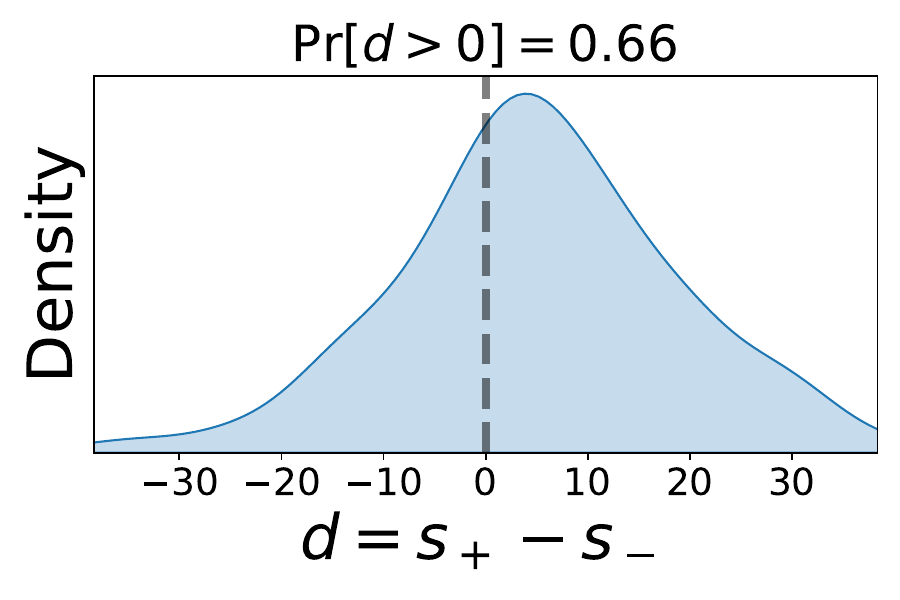}
}
\caption{Empirical Distribution of the Performance Score Change of LLM-Generated Review Evaluation. This result corresponds to \Cref{table:llm-review}.}
\label{fig:llm-review}
\end{figure}

\section{Limitation, Future Work, and Conclusion}

\paragraph{Limitations}
As the first paper, to our best knowledge, exploring eliciting subjective textual data with LLMs, our work has several limitations, and each of them may lead to a future direction. 

First, the theoretical effectiveness of the GPPM and GSPPM highly depends on how well the LLM prediction estimates the real underlying distribution (Assumption~\ref{assum:llm-prior}). The quality of this prediction can be influenced by many factors, including prompt engineering, the capacity of the LLMs, etc. The LLM predictions generated by state-of-the-art models like GPT-4 and Llama-2 may not offer perfect estimates. Nonetheless, we anticipate improvements in the efficacy of our mechanisms in the future, considering the rapid advancements in LLM technology.

Second, our empirical findings confirm that the GPPM and GSPPM can effectively penalize several degradations of the quality of agents' reports. However, the performance of our mechanisms in addressing more sophisticated manipulation strategies or even malicious strategies remains unstudied. Therefore, how to model and understand human agents' strategies in the textual world is an open question. 

Third, we primarily focused on integrating LLMs with \citet{miller2005eliciting}'s mechanism. Future directions may involve applying our method to a broader set of classic peer prediction mechanisms, especially in the multi-task setting \cite{dasgupta2013crowdsourced,shnayder2016informed, kong2019information, schoenebeck2020learning,10.1145/3543507.3583292,kong2020dominantly,Agarwal2017-ty,rowdycrowds}. It might be interesting to test whether the properties of these mechanisms in the classic setting can be generalized to the textual setting. Additionally, fine-tuning LLMs to learn the structure of agent responses in multi-task scenarios might be another promising approach.

Yet another approach is to further explore ways to use LLMs to map text into smaller dimensions.  Independent work~\citep{wu2024elicitationgpt}, explores using known high-quality texts to determine such a mapping. Their approach shares a similar intuition to our \textsc{Judgment} implementation.

Furthermore, in addition to the potent property, prior work has investigated other desiderata, including fairness \cite{burrell2021measurement} and efficiency \cite{zhang2022high,xu2024spot}. 
We view the comparisons of different text-elicitation mechanisms in terms of these properties as an interesting future work.

\paragraph{Future Work in Benchmarking LLMs with Peer Review Tasks} 
Building on our current research, there is significant potential for benchmarking LLMs in peer review tasks to mitigate data contamination and data leakage.

Evaluating LLMs on various benchmarks has become crucial in natural language processing (NLP) research. However, several studies have highlighted potential issues of data contamination and data leakage, where the testing data may have been inadvertently included in the training data for LLMs \cite{sainz2023nlp, deng2023benchmark}. These issues question the validity of current benchmarks and necessitate the exploration of alternative approaches.

One potential benchmark that mitigates data contamination is to evaluate LLMs on creating peer reviews for the latest research papers. With thousands of new research papers published annually across various fields, these papers are unlikely to appear in the pre-training data, providing a large number of novel test cases. While previous studies have evaluated LLMs' ability to generate new content, such as writing stories \cite{chakrabarty2023art}, generating crowdsourcing responses \cite{boussioux2023crowdless}, and summarizing news articles \cite{zhang2024benchmarking}, a common limitation has been the reliance on human evaluators to assess the LLM-generated content. This reliance on manual evaluation can be time-consuming, expensive, and potentially subjective, hindering the scalability and generalizability of these benchmarks.

In the context of academic peer review, existing approaches have focused on analyzing the fraction of common opinions among human and LLM-generated reviews \cite{liang2023large} or studying the correlation between LLM review scores and human scores \cite{thelwall2024can}. However, a more sophisticated benchmark is needed to comprehensively evaluate LLM performance in this domain.

To address this need, our proposed mechanism provides an automated scoring method for the quality of LLM-generated academic peer reviews. By demonstrating the ability to differentiate between reviews from humans, GPT-4, and GPT-3.5, this mechanism shows the potential to serve as an automated benchmark for peer review tasks. This approach offers several advantages, including efficient and scalable evaluation, mitigation of subjectivity associated with human evaluation, and a more standardized and reproducible framework for comparing LLM performance.

\paragraph{Conclusion} In summary, our research introduces a pioneering framework for eliciting high-quality textual judgment. To the best of our knowledge, our work is the first to design peer prediction mechanisms for eliciting high-quality textual reports. We propose two mechanisms, GPPM and GSPPM,  which utilize the LLM-derived prediction, two implementations for estimating the LLM-derived prediction, and an evaluation workflow. The use of LLM prediction could extend to other peer prediction mechanisms, given that prediction is the foundation of most peer prediction mechanisms. Our empirical results demonstrate the potential of the GPPM and GSPPM to motivate quality human-written reviews over LLM-generated reviews. 

\section*{Acknowledgments}
We would like to thank Prof. Kerby Shedden from the Consulting for Statistics, Computing, and Analytics Research (CSCAR) at the University of Michigan, as well as Huizi Yu, also from the University of Michigan, for their expert consulting on the statistical methods used in our experiments. We would also like to thank our anonymous reviewers at ACM EC'24 for their insightful feedback and suggestions.

\bibliographystyle{plainnat}
\bibliography{ref}
\clearpage
\addtocontents{toc}{\protect\setcounter{tocdepth}{0}}
\appendix

\input{appendix_proofs}
\input{appendix_additional_results}
\input{appendix_prompt}

\end{document}

%% file: new_intro.tex
\section{Introduction}\label{sec:intro}

Consider the following review for an academic paper: 
\begin{quote}
\textit{``I didn't get much out of reading this paper. Their methods do not seem very rigorous. I don't think the conclusions are supported very well.''}\end{quote}

The above review is not very informative: its initial critique is too general, and the issues with the methods and conclusions should have been explained more thoroughly. If the peer review process only gathers reviews of this quality, it will struggle to make useful and fair publication decisions. The problem has been exacerbated by large language models (LLMs), which greatly reduce the cost of generating reviews that closely mimic human-written reviews but often lack substantial insight~\cite{liang2023large}.

Here are two reviews of a randomly selected paper submitted to the ICLR\footnote{The International Conference on Learning Representations (ICLR), a top-tier machine learning conference, makes all of its peer review data openly accessible on OpenReview (https://openreview.net).} 2020. An interesting exercise is to determine which was generated by an LLM and which was written by a human.

\begin{figure}[h]
\noindent\begin{minipage}[t]{0.495\textwidth}
\begin{tikzpicture}
\node[fill=blue!7,minimum width=\textwidth,minimum height=6.9cm,text width=0.9\textwidth] (box) {\dots ~ I lean towards rejecting this paper however, because I am not convinced of the results' significance. We already know how to learn symmetric functions (see Exercise 3.26 in Mohri et al., 2018). The authors' results show that we can inject this knowledge into a neural network at initialization, and then run SGD without making things too much worse. I do not see how these ideas might apply to more substantial learning problems where our prior knowledge is less precise. \dots}; 
\end{tikzpicture}
\end{minipage}
\begin{minipage}[t]{0.495\textwidth}
\begin{tikzpicture}
\node[fill=red!7,minimum width=\textwidth,minimum height=6.9cm,text width=0.9\textwidth] (box) {\dots ~ The paper makes a valuable theoretical contribution to the understanding of neural network initialization, particularly in the context of symmetric functions. The empirical validation is a strong point, although the experiments could be expanded to provide a more comprehensive evaluation. The paper's focus on a single hidden layer network is both a strength, in terms of theoretical tractability, and a weakness, in terms of practical relevance. To move forward, the authors should consider extending their analysis to more complex architectures and providing a broader empirical evaluation.  \dots}; 
\end{tikzpicture}
\end{minipage}
\caption{An example from our study: Two reviews of a submission at ICLR2020, the left one by a human reviewer, and the right one by GPT-4.}
\label{fig:example-llm-vs-human}
\end{figure}

Due to LLMs, it is no longer possible to filter out low-quality reviews by their length, lack of any relation to the paper at hand, or poor grammatical constructions. In the above example, we can observe that the AI-generated review looks informative and effectively summarizes the paper. However, it lacks the depth and insight of the expert human review. 

The need to obtain high-quality subjective human evaluation extends beyond academia to other domains, including business, the arts, and more.
For example, if customer feedback on Amazon and Yelp is inundated with shallow reviews or those generated by LLMs, consumers will struggle with making well-informed decisions \cite{resnick2000reputation,tadelis2016reputation}. The problem of incentivizing high-quality reviews is more important and daunting than ever. %

One approach is to provide well-designed incentives for the reviewer that, in particular, reward high-quality reviews more than low-quality reviews~\cite{srinivasan2021auctions}. However, because reviews are subjective, their correctness cannot be directly verified. This inherent challenge complicates the evaluation of review quality. 

One straightforward idea is to ask other people to judge the quality of the reviews. But then we face a new challenge: how to motivate these new judges? Moreover, an automated approach that does not introduce additional participants and procedures is preferable.   

Prior work has proposed the peer prediction mechanism, a powerful tool to elicit subjective information \citep{miller2005eliciting}.
The high-level idea is to determine the reward of a person according to the ``correlation'' between her report and a peer's report. The underlying intuition is that better, more insightful reports will naturally align more closely with one another. In their setting, when a person puts in the effort to understand a task, she gains a private signal such as ``good'' or ``bad''. She can then choose whether to report this signal honestly. \citet{miller2005eliciting} prove that in their setting, truth-telling is an equilibrium---if a participant believes other people will invest effort and tell the truth, she should also do this. 

However, implementing the above mechanism requires knowledge of the prior: the joint distribution of the private signals. In the original peer prediction mechanism \cite{miller2005eliciting}, agents are asked to report their private signals such as ``good'' or ``bad''. The mechanism then predicts a peer's report, such as 70\% ``good'', by getting a posterior based on an agent's report and the prior. The agent is then rewarded for the accuracy of this posterior prediction. Assuming the mechanism has perfect knowledge of the prior, this incentivizes truth-telling because only an honest report can lead to the optimal posterior. 

The mechanism's required knowledge of the prior has been seen as a major impediment to real-world implementation of peer-prediction mechanisms. Significant advances, that follow two main approaches, have enabled overcoming this limitation in several settings. Both these approaches circumvented the requirement of knowing the prior by learning, not the prior itself, but a proxy, usually the relationship between the agent reports, from agent reports themselves.  

The first approach, often called the multitask setting, involves assigning agents to multiple a priori similar tasks. This allows learning the structure of agent reports and enables measuring the amount of information in common between agent responses~\citep{dasgupta2013crowdsourced,shnayder2016informed,kong2019information,liuchensurrogate, 10.1145/3543507.3583292}. In certain settings such as multiple choice questions, even a small number of tasks may suffice~\cite{schoenebeck2020learning,burrell2021measurement,Kong2024}. The second approach called the signal-prediction framework, pioneered by~\citet{prelec2004bayesian}, involves eliciting second-order predictions, that is asking how they believe other agents will respond, for example, ``I think 70\% of my peers will answer `good' ''~\citep{radanovic2014incentives, witkowski2012robust,schoenebeck2023two,chen2021wisdom}.  

However, because both these approaches rely on learning from agent responses or predictions, they work better when the space is simple---either categorical (such as a multi-choice question) or numerical (such as a rating between 0 and 10). Otherwise, the structure is too involved to learn in the multitask setting, and the forecasts can not be communicated efficiently for forecast elicitation.

However, reducing to such a simple space often loses the rich information within the textual judgments. For example, in peer review, the decision of the editor/area chair often relies more on the arguments and justifications in the textual reviews rather than merely on numerical ratings. Furthermore, on online platforms, the inflation of ratings makes them less reliable and distinguishable, while textual reviews tend to be more stable \cite{filippas2018reputation}.  

Given these limitations and the recent success of large language models (LLMs), our research question is: \textbf{can we develop automated mechanisms that effectively incentivize high-quality, informative textual feedback by rewarding it more than generic or low-quality content?}

Intuitively, eliciting textual feedback is inherently more difficult than eliciting numerical or categorical responses. However, the recent rise of powerful Large Language Models (LLMs) has surprisingly flipped this script. 
LLMs, more or less, estimate the probability distribution of the entirety of human language. Thus, our goal is to instead run the original peer-prediction mechanism by using LLMs to gain access to the prior.  
We use the LLM's ability to analyze the structure of textual responses and predict the probability of one text given another text (LLM-prediction). This directly addresses the ``knowing prior'' problem and eliminates the need for multiple tasks or second-order predictions (common for categorical/numerical responses). It is somewhat paradoxical that moving to this much \emph{more} complex domain actually may make the entire task easier! In essence, LLMs make eliciting textual responses easier than simpler formats, as textual responses offer more complexity that LLMs can leverage, while simpler formats lack this richness. 

Directly employing LLMs may reward superficial similarities, such as matching speaking styles, or reviews that offer no more than a reiteration of the paper's abstract, which may benefit LLM-generated reviews. Ultimately, the goal is to encourage reviewers to delve deeper, providing unique perspectives. A related question is \textbf{can we distinguish the valuable, unique human expert reviews from the coherent yet potentially superficial reviews generated by AI? }

To answer the above research question and not reward superficial similarities, when using LLMs to compute the correlation, it is important to effectively condition out ``shortcut'' information such as language styles and information contained in any synopsis of the reviewed item. We borrow the term from ``shortcut learning'', where a machine learning algorithm learns the undesired information from data that is strongly correlated with labels on the training data but lacks generality \cite{geirhos2020shortcut}. In our setting, for example, a human-written review can have a high correlation with an LLM-generated review because they mentioned several particular terms in the paper. However, such ``shortcut'' information may lead to unintended rewards for shallow reviews (e.g., LLM-generated reviews) and noise caused by different language styles. By conditioning out ``shortcut'' information, we aim to filter out these superficial aspects and focus on rewarding reviews that demonstrate a deeper level of engagement.

\subsection{Our Contribution}
We apply the LLM-prediction to peer prediction and propose two mechanisms---the \textsc{Generative Peer Prediction mechanism (GPPM)} and the \textsc{Generative Synopsis Peer Prediction mechanism (GSPPM)}. At a high level, the former rewards a review based on how much it helps predict the contents of another review. The latter, however, rewards a review based on how much \emph{more} it helps predict the contents of another review than a mere synopsis of the item to be reviewed, such as the abstract of a paper, thereby conditioning out the "shortcut" information derived from superficial information contained in the synopsis.

We first use theory to present the main idea of our method, where we consider a model with three layers of effort: high, low, and no effort. We show that when the LLM prediction is sufficiently accurate, both mechanisms can incentivize high effort and truth-telling as an (approximate) Bayesian Nash equilibrium. By conditioning on a synopsis of the item to be reviewed, GSPPM can further shrink the gap of expected scores between low-effort and no-effort reporting while preserving the gap between high-effort and low-effort reporting. 

We then propose two heuristic implementations for getting the LLM-prediction---\textsc{Token} and \textsc{Judgment}. We highlight that for a robust implementation, it is necessary to preprocess the responses. We propose a straightforward yet effective heuristic preprocessing method. This involves using an LLM to rephrase and summarize the initial responses, aiming to standardize the language style and remove superficial content, and consequently, mitigate the impact of the ``shortcut''.

Moreover, we conduct experiments on the mechanisms using an ICLR dataset and a Yelp review dataset and observed the following results:

\begin{description}
    
    \item[Result 1:] \textbf{GPPM can effectively penalize report degradations.} We apply three report degradation methods, which degrade the information of an agent's report. In both ICLR and Yelp datasets, we observe that the expected score computed by the GPPM significantly decreases after all degradations.

    \item[Result 2:]  \textbf{Both GPPM and GSPPM can differentiate three quality levels---human, GPT-4, and GPT-3.5.} In the ICLR dataset, we replace an agent's review with a GPT-4-generated review and a GPT-3.5-generated review respectively, representing a decreasing level of effort. We observe that the expected scores computed by both GPPM and GSPPM significantly decrease. Furthermore, the decrease of the GPT-3.5-generated review is larger than that of the GPT-4-generated review. 

    \item[Result 3:]  \textbf{GSPPM penalizes LLM-generated peer review more than GPPM.} We find that the GSPPM applies a more significant expected score penalty on the LLM-generated peer reviews, including both GPT-4 and GPT3.5, compared to GPPM, which indicates its improved capacity to distinguish high-quality reports from low-quality reports.
    
\end{description}

We further note that our mechanisms can serve not only to assess the quality of reviews which can inform decision-making, but also to incentivize effort from agents. This can be interpreted by considering GPT-3.5, GPT-4, and human-written reviews as representing three levels of effort. Thus, by rescaling the scores of the mechanisms into payments, we can reward high-effort reviews much more than low-effort reviews.

%% file: appendix_proofs.tex
\section{Omitted Proofs}
\label{app:proofs}

\propGPPM*

\begin{proof}[Proof of Proposition~\ref{prop:GPPM}]
For agent $i$, we analyze her expected score under each strategy given that agent $j$ exerts high effort and reports truthfully. 

If agent $i$ also exerts high effort and reports truthfully, her expected score is 
\begin{align*}
&\sum_{x_j}\Pr[X_j=x_j \mid X_i=x_i]\log \Pr_{\LLM(\psi)}[X_j=x_j \mid X_i=x_i]\tag{\Cref{assum:llm-prior}}\\
\geq & \sum_{x_j}\Pr[X_j=x_j \mid X_i=x_i]\log \Pr[X_j=x_j \mid X_i=x_i]-\epsilon\\
= &-H(X_j\mid X_i=x_i)-\epsilon,
\end{align*}

Taking the expectation of $X_i$, her expected score is

\begin{align*}
&\sum_{x_i} \Pr[X_i=x_i] \sum_{x_i}\Pr[X_j=x_j \mid X_i=x_i]\log \Pr_{\LLM(\psi)}[X_j=x_j \mid X_i=x_i]\\
\geq &\sum_{x_i} \Pr[X_i=x_i] \left(\sum_{x_j}\Pr[X_j=x_j \mid X_i=x_i]\log \Pr[X_j=x_j \mid X_i=x_i]-\epsilon\right)\tag{\Cref{assum:llm-prior}}\\
= & \sum_{x_i, x_j}\Pr[X_i=x_i,X_j=x_j]\log \Pr[X_j=x_j \mid X_i=x_i]-\epsilon\\
= &-H(X_j \mid X_i)-\epsilon,
\end{align*}
where $H(X_j \mid X_i)$ is the conditional entropy. 


If agent $i$ exerts high effort but reports non-truthfully $\sigma\neq \tau$, her expected score is 
\begin{align*}
& \sum_{x_j}\Pr[X_j=x_j\mid X_i=x_i]\log \Pr_{\LLM(\psi)}[X_j=x_j\mid X_i = \strategy(x_i)]\\
\leq & \sum_{x_j}\Pr[X_j=x_j\mid X_i=x_i]\log \Pr[X_j=x_j\mid X_i=x_i] \tag{Log scoring rule is proper.}\\
=& -H(X_j\mid X_i=x_i).
\end{align*}

With the stochastic relevance assumption and the fact that LSR is strictly proper, the inequality is strict when $x_i \neq \strategy(x_i)$. Notice that there must exist $x_i$ such that $x_i \neq \strategy(x_i)$ with positive probability as $\sigma\neq \tau$. Therefore, taking the expectation of $X_i$ and $\strategy(X_i)$, her expected score is striclty less than $-H(X_j\mid X_i)$. Thus, if agent $i$ also exerts high effort, reporting truthfully will be at least $-\epsilon$ better than reporting non-truthfully. 

If agent $i$ exerts low effort and observes $\Lowsig_i$, and she reports truthfully or non-truthfully with $\strategy(\Lowsig_i)$, then with an analogous derivation, her expected score is 
\begin{align*}
&\sum_{\lowsig_i} \Pr[\Lowsig_i=\lowsig_i]\sum_{x_j}\Pr[X_j=x_j \mid \Lowsig_i=\lowsig_i]\log \Pr_{\LLM(\psi)}[X_j=x_j \mid X_i=\strategy(\lowsig_i)]\leq -H(X_j \mid \Lowsig_i).
\end{align*}

This inequality is based on the fact that Log scoring rule is proper. 

Notice that according to our hierarchical effort model, $H(X_j \mid \Lowsig_i)<H(X_j \mid X_i)$ as $\Lowsig_i$ is determined by $X_i$. Therefore, the difference in the expected scores between exerting high effort and reporting truthfully compared to exerting low effort is at least $H(X_j \mid \Lowsig_i)-H(X_j \mid X_i)-\epsilon=I(X_i;X_j \mid \Lowsig_i)-\epsilon$.

With analogous analyses, the difference in the expected scores from investing in no-effort compared to high-effort is at least $I(X_i;X_j)-\epsilon=H(X_j )-H(X_j \mid X_i)-\epsilon\geq H(X_j \mid \Lowsig_i)-H(X_j \mid X_i)-\epsilon=I(X_i;X_j \mid \Lowsig_i)-\epsilon$.

\end{proof}

\propGSPPM*

\begin{proof}[Proof of Proposition~\ref{prop:GSPPM}]

In GSPPM, with analogous analyses, if agent $i$ also exerts high effort and reports truthfully, her expected score is at least $-H(X_j \mid X_i=x_i,\Synopsis=\synopsis)-\epsilon'$. Taking the expectation over $X_i$ and $\Synopsis$, the expected score is $-H(X_j \mid X_i,\Synopsis)-\epsilon'$. If agent $i$ also exerts high effort and reports non-truthfully, her expected score is strictly less than $-H(X_j \mid X_i=x_i,\Synopsis=\synopsis)$. Taking the expectation over $X_i$ and $\Synopsis$, the expected score is strictly less than $-H(X_j \mid X_i,\Synopsis)$.

If agent $i$ exerts low effort and observes $\Lowsig_i$, and she reports truthfully or non-truthfully with $\strategy(\Lowsig_i)$ her expected score is 
\begin{small}\begin{align*}
&\sum_{\lowsig_i} \Pr[\Lowsig_i=\lowsig_i]\sum_{\synopsis}\Pr[\Synopsis=\synopsis|\Lowsig_i=\lowsig_i]\sum_{x_j}\Pr[X_j=x_j \mid \Lowsig_i=\lowsig_i,\Synopsis=\synopsis]\log \Pr_{\LLM(\psi)}[X_j=x_j \mid X_i=\strategy(\lowsig_i),\Synopsis=\synopsis]\\

=&\sum_{\lowsig_i,\synopsis} \Pr[\Lowsig_i=\lowsig_i,\Synopsis=\synopsis]\sum_{x_j}\Pr[X_j=x_j \mid \Lowsig_i=\lowsig_i,\Synopsis=\synopsis]\log \Pr_{\LLM(\psi)}[X_j=x_j \mid X_i=\strategy(\lowsig_i),\Synopsis=\synopsis]\\

\leq&\sum_{\lowsig_i,\synopsis} \Pr[\Lowsig_i=\lowsig_i,\Synopsis=\synopsis]\sum_{x_j}\Pr[X_j=x_j \mid \Lowsig_i=\lowsig_i,\Synopsis=\synopsis]\log \Pr[X_j=x_j \mid \Lowsig_i=\lowsig_i,\Synopsis=\synopsis]\\

= &-H(X_j \mid \Lowsig_i,\Synopsis) 
\end{align*}\end{small}

The inequality is based on the fact that LSR is proper.

Thus, if agent $i$ also exerts low effort, her expected score is at most $-H(X_j \mid \Lowsig_i,\Synopsis)$. Therefore, the difference in the expected scores between exerting high-effort compared to low-effort is at least $H(X_j \mid \Lowsig_i,\Synopsis)-H(X_j \mid X_i,\Synopsis)-\epsilon'=I(X_i;X_j \mid \Synopsis,\Lowsig_i)-\epsilon'$. 

With analogous analyses, the difference in the expected scores from investing in no-effort compared to high-effort is at least $I(X_i;X_j \mid \Synopsis)-\epsilon'=H(X_j \mid \Synopsis)-H(X_j \mid X_i,\Synopsis)-\epsilon'\geq H(X_j \mid \Lowsig_i,\Synopsis)-H(X_j \mid X_i,\Synopsis)-\epsilon'=I(X_i;X_j \mid \Synopsis,\Lowsig_i)-\epsilon'$.

Therefore, whenever $I(X_i;X_j \mid \Synopsis)>\epsilon'$, exerting high effort and reporting truthfully is an $\epsilon'$-Bayesian Nash equilibrium in GSPPM.
\end{proof}

\propGPPMpotent*

\begin{proof}[Proof of Proposition~\ref{prop:GPPM-potent}]

\Cref{prop:GPPM} lower bounds the gap of the \emph{expected performance score} between (truth-telling, high effort) and any strategy of an agent in three cases. To prove that the mechanism is $\delta$-potent, we have to show that there exist constants $\alpha>0,\beta$ such that the gap of agent's \emph{expected utility} between (truth-telling, high effort) and any strategy in the above three cases is always lower-bounded by $-\delta$. 

\begin{itemize}
    \item Suppose agent $i$ plays $(\sigma_i, c_h)$ where $\sigma_i$ is an arbitrary reporting strategy. In this case, 
    \begin{align*}
        &U_i((\tau, c_h), (\tau, c_h)) - U_i((\sigma_i, c_h), (\tau, c_h)) \\
        =& \alpha\left(\mathbb{E}[\textsc{GPPM}(\Signal_i, \Signal_j)] - \mathbb{E}[\textsc{GPPM}(\strategy(\Signal_i), \Signal_j)]\right) \\
        \ge& -\alpha \epsilon. \tag{\Cref{prop:GPPM}}
    \end{align*}
    \item Suppose agent $i$ plays $(\sigma_i, c_l)$. In this case, 
    \begin{align*}
        &U_i((\tau, c_h), (\tau, c_h)) - U_i((\sigma_i, c_l), (\tau, c_h)) \\
        =& \alpha\left(\mathbb{E}[\textsc{GPPM}(\Signal_i, \Signal_j)] - \mathbb{E}[\textsc{GPPM}(\strategy(\Signal^l_i), \Signal_j)]\right) - c_h + c_l \\
        \ge& \alpha I(\Signal_i;\Signal_j\mid \Lowsig_i) -\alpha \epsilon - (c_h-c_l). \tag{\Cref{prop:GPPM}}
    \end{align*}
    \item Suppose agent $i$ plays $(\sigma_i, 0)$. In this case, 
    \begin{align*}
        &U_i((\tau, c_h), (\tau, c_h)) - U_i((\sigma_i, 0), (\tau, c_h)) \\
        =& \alpha\left(\mathbb{E}[\textsc{GPPM}(\Signal_i, \Signal_j)] - \mathbb{E}[\textsc{GPPM}(\strategy(\textsc{Null}), \Signal_j)]\right) - c_h \\
        \ge& \alpha I(\Signal_i;\Signal_j) -\alpha \epsilon - c_h. \tag{\Cref{prop:GPPM}}
    \end{align*}
\end{itemize}

Take these together, we have that when
$\alpha = \max\left\{\frac{c_h-c_l}{I(\Signal_i;\Signal_j\mid \Lowsig_i)}, \frac{c_h}{I(\Signal_i;\Signal_j)}\right\}$
we can obtain a unified lower bound $\delta = \alpha \epsilon$. This completes the proof.

\end{proof}

\propGSPPMpotent*

\begin{proof}[Proof of Proposition~\ref{prop:GSPPM-potent}]
With analogous analyses of proof of Proposition~\ref{prop:GPPM-potent}, we have Proposition~\ref{prop:GSPPM-potent}.
\end{proof}

\coroGAP*

\begin{proof}[Proof of Corollary~\ref{coro:gap_bound}]
When $\Lowsig_i=\Lowsig_j=g(\Synopsis)$, $\epsilon=\epsilon'$, given the fact that $\Lowsig_i$ and $\Lowsig_j$ contain partial information of $X_i$ and $X_j$ respectively, and the assumption that the synopsis does not reveal more information about $X_i$ than $\Lowsig_i$ and $\Lowsig_j$, the statements of the formulas for Gaps in Proposition~\ref{prop:GPPM} and Proposition~\ref{prop:GSPPM} directly imply the results. 
\end{proof}

%% file: appendix_additional_results.tex
\section{Additional Results}\label{appendix:additional}

In this section, we demonstrate the additional results of our experiments, which are omitted in the main text. 

\subsection{\textsc{Token-Raw} v.s.~\textsc{Token-Preprocess}}\label{sec:need-for-preprocess}

In this subsection, we delve into the effectiveness of preprocessing by contrasting the performance of \textsc{Token-Raw} and \textsc{Token-Preprocess}. Intuitively, \textsc{Token-Preprocess} provides scores reflecting the semantic quality better since it removes `shortcut'' information confounding the LLM prediction, such as the paper summary or the reviewer's language style. We present and explain our observations as follows.

First, we observe that \textsc{Token-Raw} outperforms \textsc{Token-Preprocess} on two degradation tasks (Figure~\ref{fig:bar-chart-token-raw-heuristic}). We conjecture that the success of \textsc{Token-Raw} relies on the superficial information commonly found in different reviews of the same item. One example is the paper summary, which most reviewers write before considering the pros and cons. When we use one review to predict tokens in another review of the same item, such common superficial information can significantly increase the performance score. Therefore, as the random replacement degradation and sentence-level degradation remove or reduce the superficial information, \textsc{Token-Raw} can capture such a change more effectively than \textsc{Token-Preprocess}, since the latter removes the superficial information in both reviews.

\begin{figure}[ht]
\centering
\subfigure[Results on ICLR Dataset]{\centering
\includegraphics[scale=0.29]{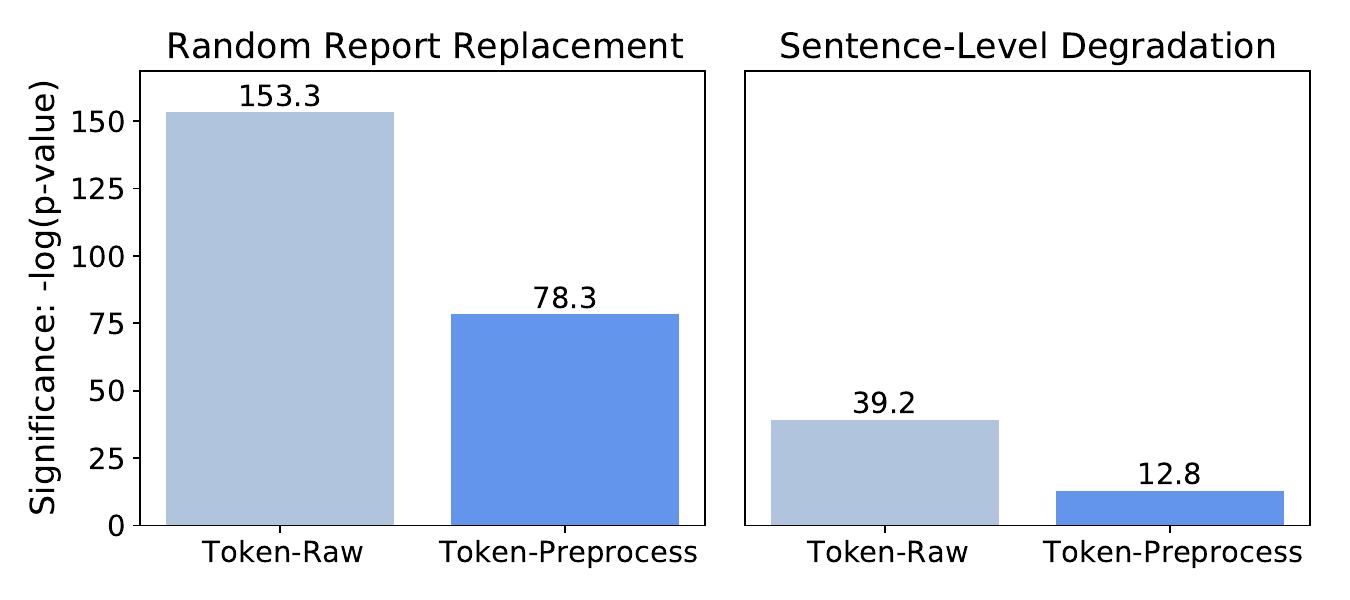}
}
\subfigure[Results on Yelp Dataset]{\centering
\includegraphics[scale=0.29]{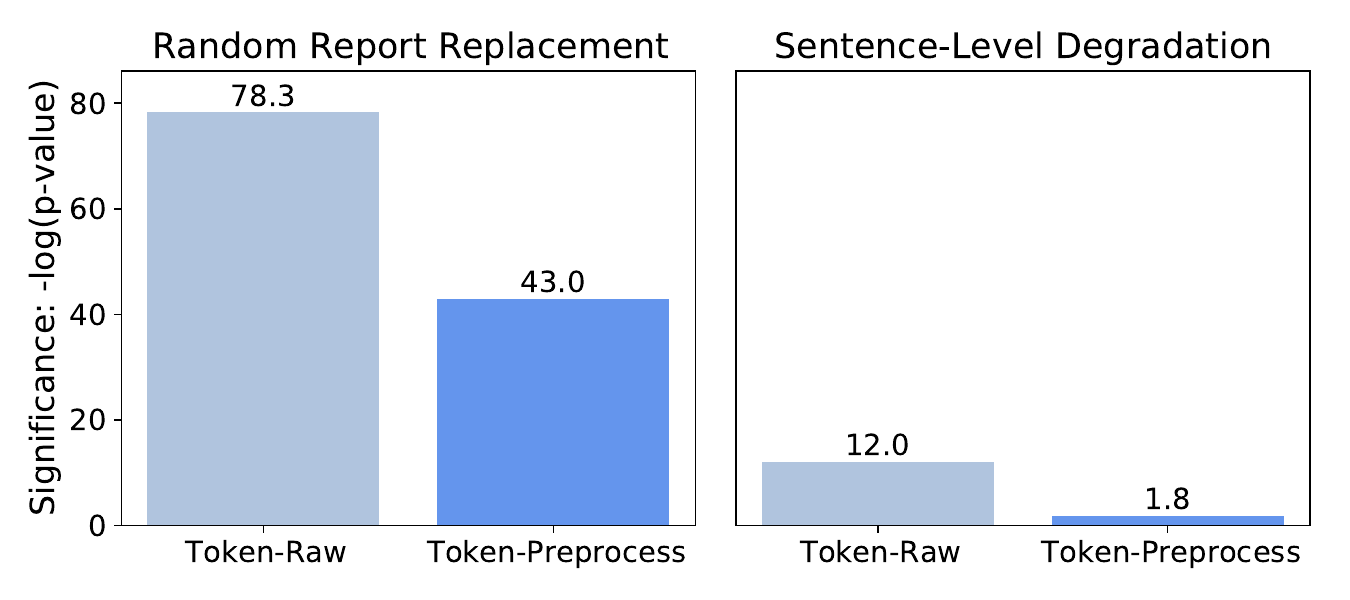}
}
\caption{Report Degradation Evaluation Result for GPPM with \textsc{Token-raw} and \textsc{Token-preprocess} Implementations:  \textbf{Significance ($\bm{-\log_{10}}$(p-value))} of the expected score difference $\bar{d}>0$, \textbf{higher is better}. Typically, a significance score $-\log_{10}(\text{p-value})> 1.30$ (equivalent to p-value $< 0.05$) is regarded as significant difference. We defer the comprehensive statistics metrics to Table~\ref{table:random-replace} and Table~\ref{table:sentence-degrade}.}
\label{fig:bar-chart-token-raw-heuristic}
\end{figure}

Our conjecture is further confirmed by the fact that the performance of GPPM implemented with \textsc{Token-Raw} significantly drops when facing more complex tasks like distinguishing GPT reviews from human reviews, especially for GPT-4 generated reviews (Figure~\ref{fig:bar-chart-LLM-review-GPPM}). As GPT can successfully generate reviews with plausible superficial information, we conjecture this is because the superficial information can successfully fool \textsc{Token-Raw}. 

\begin{figure}[ht]
\centering
\subfigure[GPPM]{\centering
\includegraphics[scale=0.29]{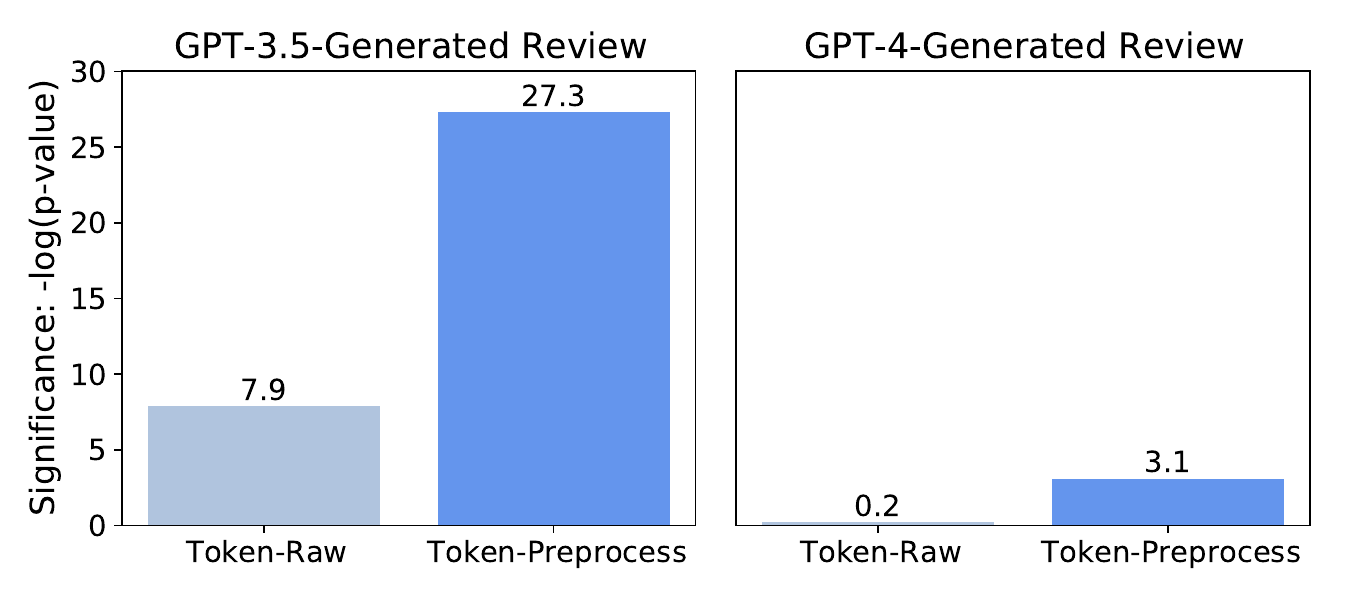}
\label{fig:bar-chart-LLM-review-GPPM}
}
\subfigure[GSPPM]{\centering
\includegraphics[scale=0.29]{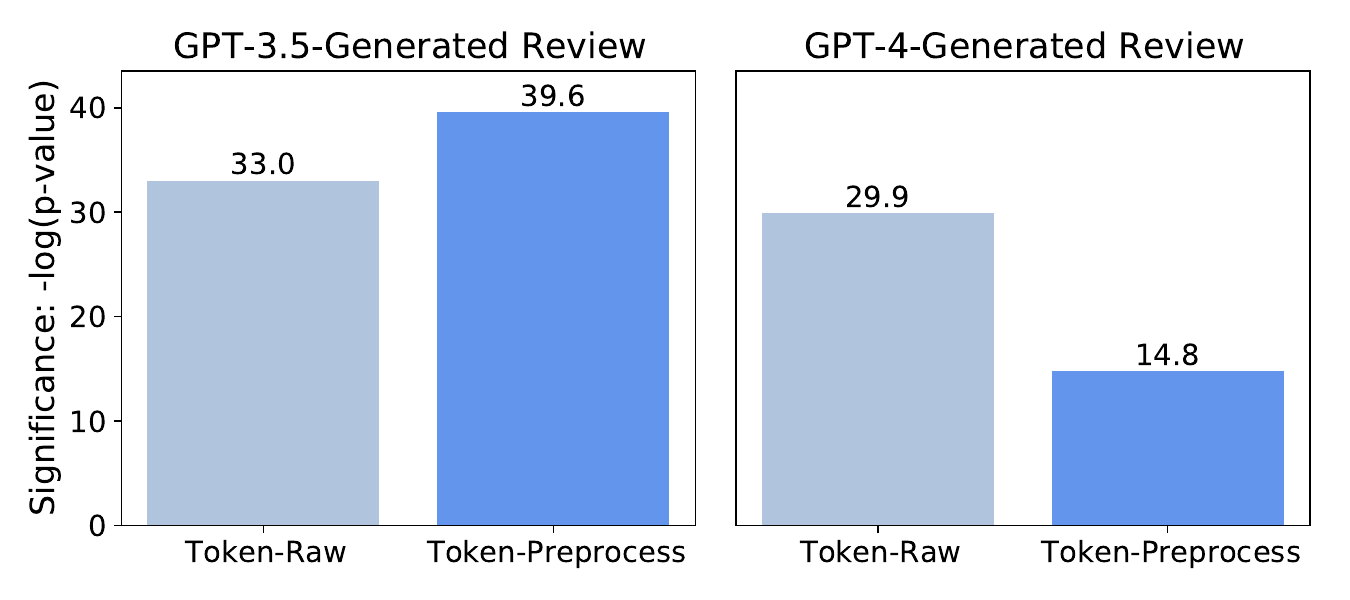}
\label{fig:bar-chart-LLM-review-GSPPM}
}
\caption{Evaluation Result for GPPM and GSPPM with \textsc{Token-raw} and \textsc{Token-preprocess} Implementations in LLM-generated Review:  \textbf{Significance ($\bm{-\log_{10}}$(p-value))} of the expected score difference $\bar{d}>0$, \textbf{higher is better}. Typically, a significance score $-\log_{10}(\text{p-value})> 1.30$ (equivalent to p-value $< 0.05$) is regarded as significant difference. The comprehensive statistics metrics are in Table~\ref{table:llm-review} and Table~\ref{table:llm-review-raw}.}
\label{fig:bar-chart-LLM-review}
\end{figure}

\begin{table}[htbp]
\small\centering\renewcommand{\arraystretch}{0.6}
\begin{tabular}{cccccccc}
\toprule
Mechanism & Review-Generating LLM & $\bar{d}$ & $\sigma(d)$ & $SE(\bar{d})$ & t-statistic & p-value & $-\log_{10}(\text{p-value})$\\ 
\cmidrule{1-8}
\multirow{2.5}{*}{GPPM}  & GPT-3.5 & 8.074 & 31.891 & 1.426 & 5.656 & 1.3e-08 & 7.884\\
\cmidrule{2-8}
                       & GPT-4 & -0.511 & 31.689 & 1.417 & -0.360 & 6.4e-01 & 0.193 \\
\cmidrule{1-8}
\multirow{2.5}{*}{GSPPM} & GPT-3.5 & 13.776 & 23.687 & 1.059 & 12.992 & 9.5e-34 & 33.024\\
\cmidrule{2-8}
                       & GPT-4 & 12.912 & 23.557 & 1.054 & 12.244 & 1.3e-30 & 29.898\\
\bottomrule
\end{tabular}
\caption{Statistics Metrics of LLM-Generated Review Evaluation for GPPM/GSPPM implemented with \textsc{Token-Raw}. Note that GPPM with \textsc{Token-raw} has a p-value $0.64 > 0.05$, indicating a failure in differentiating GPT-4-generated review from human-written review. }
\label{table:llm-review-raw}
\end{table}

Furthermore, when we condition on the synopsis (papers' abstract) using GSPPM, \textsc{Token-Preprocess} can again successfully distinguish GPT-generated reviews from human-written reviews. This is because the superficial information is primarily included in the synopsis, conditioned out by the mechanism. However, in \Cref{table:llm-review-raw}, we observe that replacing human reviews with GPT-3.5 or GPT-4 reviews leads to a similar expected score decrease, indicating that GSPPM with \textsc{Token-Raw} cannot distinguish GPT-3.5 reviews from GPT-4 reviews as effectively as GSPPM with \textsc{Token-Preprocess} (as shown in \Cref{sec:result2,sec:result3}). We conjecture this is because the language style greatly influences LLMs' predictions on raw reviews.\footnote{Previous prompt engineering studies have provided evidence that the language style in the prompt can impact LLMs' output \cite{arora2022ask}} As GPT-3.5 and GPT-4 have similar language styles but are very different from human reviewers, \textsc{Token-Raw} fails to separate the two types of GPT-written reviews well though it succeeds in separating GPT reviews from human reviews.

\subsection{Statistics Metrics of Evaluations}

In this subsection, we provide all statistics metrics omitted in the main text (\Cref{table:random-replace,table:sentence-degrade,table:judgment-degrade}).

\begin{table}[htbp]
\small\centering\renewcommand{\arraystretch}{0.6}
\begin{tabular}{cccccccc}
\toprule
Dataset & Implementation & $\bar{d}$ & $\sigma(d)$ & $SE(\bar{d})$ & t-statistic & p-value & $-\log_{10}(\text{p-value})$\\ 
\midrule
\multirow{5.25}{*}{ICLR} & \textsc{Token-Raw} & 55.444 & 31.736 & 1.419 & 39.026 & 5.0e-154 & 153.305\\
\cmidrule{2-8}
& \textsc{Token-Preprocess} & 20.784 & 20.483 & 0.916 & 22.666 & 5.0e-79 & 78.298\\
\cmidrule{2-8}
                      & \textsc{Judgment} & 4.522 & 8.813 & 0.394 & 11.462 & 1.8e-27 & 26.735\\
\cmidrule{2-8}
                      & baseline & 0.207 & 0.613 & 0.027 & 7.532 & 1.2e-13 & 12.928\\                      
\cmidrule{1-8}  
\multirow{5.25}{*}{Yelp} & \textsc{Token-Raw} & 9.975 & 15.323 & 0.485 & 20.576 & 5.2e-79 & 78.284\\
\cmidrule{2-8}
& \textsc{Token-Preprocess} & 5.007 & 10.915 & 0.345 & 14.500 & 1.1e-43 & 42.955\\
\cmidrule{2-8}
                      & \textsc{Judgment} & 1.055 & 5.438 & 0.172 & 6.131 & 6.3e-10 & 9.202\\
\cmidrule{2-8}
                      & baseline & 0.018 & 0.195 & 0.006 & 2.982 & 1.5e-03 & 2.834\\    
\bottomrule
\end{tabular}
\caption{Statistics Metrics of Random Report Replacement Evaluation for GPPM. $\bar{d}$ represents the mean of the score differences, $\sigma(d)$ represents the standard deviation of the score differences, and $SE(\Bar{d})$ represents the standard error of the mean difference.
}
\label{table:random-replace}
\end{table}

\begin{table}[htbp]
\small\centering\renewcommand{\arraystretch}{0.6}
\begin{tabular}{cccccccc}
\toprule
Dataset & Implementation & $\bar{d}$ & $\sigma(d)$ & $SE(\bar{d})$ & t-statistic & p-value & $-\log_{10}(\text{p-value})$\\ 
\midrule
\multirow{4}{*}{ICLR}  & \textsc{Token-Raw} & 11.192 & 17.337 & 0.775 & 14.421 & 5.8e-40 & 39.233\\
\cmidrule{2-8}
& \textsc{Token-Preprocess} & 5.117 & 15.261 & 0.683 & 7.490 & 1.6e-13 & 12.803\\
\cmidrule{2-8}
                        & \textsc{Judgment} & 1.662 & 5.396 & 0.241 & 6.880 & 9.0e-12 & 11.046\\
\cmidrule{1-8}
\multirow{4}{*}{Yelp}  & \textsc{Token-Raw} & 1.778 & 7.873 & 0.249 & 7.137 & 9.2e-13 & 12.038\\
 \cmidrule{2-8}
& \textsc{Token-Preprocess} & 0.497 & 7.316 & 0.231 & 2.146 & 1.6e-02 & 1.794\\
\cmidrule{2-8}
                        & \textsc{Judgment} & 0.148 & 3.131 & 0.099 & 1.494 & 6.8e-02 & 1.169\\
\bottomrule
\end{tabular}
\caption{Statistics Metrics of Sentence-Level Degradation Evaluation for GPPM.}
\label{table:sentence-degrade}
\end{table}

\begin{table}[htbp]
\small\centering\renewcommand{\arraystretch}{0.6}
\begin{tabular}{cccccccc}
\toprule
Dataset & Implementation & $\bar{d}$ & $\sigma(d)$ & $SE(\bar{d})$ & t-statistic & p-value & $-\log_{10}(\text{p-value})$\\ 
\midrule
\multirow{2.5}{*}{ICLR} & \textsc{Token-Preprocess} & 8.860 & 10.487 & 0.469 & 18.872 & 1.2e-60 & 59.928\\
\cmidrule{2-8}
                        & \textsc{Judgment} & 2.150 & 5.133 & 0.230 & 9.357 & 1.4e-19 & 18.857\\
\cmidrule{1-8}
\multirow{2.5}{*}{Yelp} & \textsc{Token-Preprocess} & 0.399 & 5.606 & 0.177 & 2.252 & 1.2e-02 & 1.911\\
\cmidrule{2-8}
                        & \textsc{Judgment} & 0.362 & 2.524 & 0.080 & 4.534 & 3.2e-06 & 5.489\\
\bottomrule
\end{tabular}
\caption{Statistics Metrics of Judgment-Level Degradation Evaluation for GPPM.}
\label{table:judgment-degrade}
\end{table}

\subsection{Evaluation Results of Implementation \textsc{Judgment} with Llama-2}

As we discussed in \Cref{sec:result1}, here we present the evaluation results for the implementation \textsc{Judgment} with Llama-2 (referred to as \textsc{Judgement-Llama-2}), detailed in \Cref{table:random-replace-judgment-llama-2,table:sentence-degrade-judgment-llama-2}. Since Llama-2 is a weaker model compared to GPT-4 and exhibits issues with adherence to prompt instructions, the results are as expected that \textsc{Judgement-Llama-2} yields worse results in both evaluations.

\begin{table}[htbp]
\small\centering\renewcommand{\arraystretch}{0.6}
\begin{tabular}{cccccccc}
\toprule
Dataset & Implementation & $\bar{d}$ & $\sigma(d)$ & $SE(\bar{d})$ & t-statistic & p-value & $-\log_{10}(\text{p-value})$\\ 
\midrule
ICLR & \textsc{Judgement-Llama-2} & 0.815 & 10.634 & 0.336 & 2.422 & 7.9e-03 & 2.103\\             
\cmidrule{1-8}  
Yelp & \textsc{Judgement-Llama-2} & 0.513 & 5.474 & 0.173 & 2.962 & 1.6e-03 & 2.806\\
\bottomrule
\end{tabular}
\caption{Statistics Metrics of Random Report Replacement Evaluation (\textsc{Judgement-Llama-2})}
\label{table:random-replace-judgment-llama-2}
\end{table}

\begin{table}[htbp]
\small\centering\renewcommand{\arraystretch}{0.6}
\begin{tabular}{cccccccc}
\toprule
Dataset & Implementation & $\bar{d}$ & $\sigma(d)$ & $SE(\bar{d})$ & t-statistic & p-value & $-\log_{10}(\text{p-value})$\\ 
\midrule
ICLR & \textsc{Judgement-Llama-2} & -0.246 & 11.152 & 0.353 & -0.697 & 7.6e-01 & 0.121\\             
\cmidrule{1-8}  
Yelp & \textsc{Judgement-Llama-2} & 0.407 & 4.451 & 0.141 & 2.890 & 2.0e-03 & 2.706\\
\bottomrule
\end{tabular}
\caption{Statistics Metrics of Sentence-Level Degradation Evaluation (\textsc{Judgement-Llama-2})}
\label{table:sentence-degrade-judgment-llama-2}
\end{table}

\section{Alternative Implementation Based on Clustering}\label{appendix:alter}

In this section, we explore an alternative implementation named $\textsc{Cluster}$. This implementation employs clustering of judgments to estimate $\Pr[\Signal_j = \report_j \mid \Signal_i = \report_i]$. The rationale behind the \textsc{Cluster} implementation is to address situations where LLMs lack specific task knowledge, resulting in inaccurate log probabilities or judgment predictions. 

\paragraph{Implementation}
We present the pseudocode of the $\textsc{Cluster}$ implementation in \Cref{alg:cluster}.
Similar to the \textsc{Judgment} implementation (\Cref{sec:method2}), 
we reload each report $\report$ as a set of judgments. In the preparation, we structure the dataset as a list of paired sets of judgments $D = \{(\report_{i_t},\report_{j_t})\}_{t\in [N_{D}]}$, where each pair $(\report_{i_t},\report_{j_t})$ indicates two reviews of the same item. Then, we leverage all the judgments in $D$ to fine-tune a short-text embedder, applying the technique introduced by \citet{zhang-etal-2023-clusterllm}. This embedder takes a judgment as input and generates a high-dimensional vector as its embedding, which allows us to employ the Minibatch K-means algorithm to build a clustering structure to classify judgments into a fixed number $N_c$ of clusters\footnote{We set $N_c=30$ in the implementation.}. Consequently, We introduce $\report[k]$ to denote a binary cluster indicator for the existence of a judgment of cluster $k\in [N_c]$ in $\report$, and $\Signal[k]$ as the random variable indicating the existence of a judgment of cluster $k$ in $\Signal$.

To calculate the score of \textsc{Generative Peer Prediction Mechanism}, it is necessary to estimate the conditional probability $\Pr[\Signal_j = \report_j \mid \Signal_i = \report_i]$. We make two assumptions to simplify the estimation. 
\begin{enumerate}
    \item First, we assume the clusters capture all the information a review has. This means the joint distribution between two reviews can be represented as the joint distribution between all $\Signal_i[k]$ and $\Signal_j[k]$, 
    \[\Pr[\Signal_i = \report_i, \Signal_j = \report_j] = \Pr\left[\Signal_i[1] = \report_i[1], \Signal_j[1] = \report_j[1], \ldots, \Signal_i[N_c] = \report_i[N_c], \Signal_j[N_c] = \report_j[N_c]\right].\]
    \item Second, we assume indicators of different clusters are independent. That is, for any subset $\mathcal{K} \subseteq [N_c]$, such that,

    \begin{align*}
    &\Pr\left[\bigcap_{k \in \mathcal{K}} \Signal_i[k] = \report_i[k], \Signal_j[k] = \report_j[k]\right]
    = \prod_{k \in \mathcal{K}} \Pr\left[\Signal_i[k] = \report_i[k], \Signal_j[k] = \report_j[k]\right]
    \end{align*}
    
\end{enumerate}

 With these two assumptions, we can compute the conditional probability as

\begin{align*}
\Pr[\Signal_j = \report_j \mid \Signal_i = \report_i] = & \frac{\Pr[\Signal_i = \report_i, \Signal_j = \report_j]}{\Pr[ \Signal_i = \report_i]}\\
= & \frac{\prod_{k\in [N_c]} \Pr[\Signal_i[k]=\report_i[k], \Signal_j[k]=\report_j[k]]}{\prod_{k\in [N_c]} \Pr[\Signal_i[k]=\report_i[k]]}.
\end{align*}

Furthermore, the numerator can be estimated by the empirical frequency, 
\begin{align*}
\Pr[\Signal_i[k]=\report_i[k], \Signal_j[k]=\report_j[k]]\approx & \frac{1}{|D|}\sum_{(\report_{i_t},\report_{j_t})\in D} \mathbf{1}[\report_{i_t}[k]=\report_i[k], \report_{j_t}[k]=\report_j[k]].
\end{align*}

\begin{algorithm}[H]

\caption{$\textsc{Cluster}$: Peer Prediction Score by Clustering judgments}
\label{alg:cluster}

\BlankLine

\SetKwFunction{Initialize}{Initialize}
\SetKwFunction{Query}{Query}

\textbf{Initialization:} \Initialize{$D$}\;
\KwIn{Review dataset $D = \{(\report_{i_t},\report_{j_t})\}_{t\in [N_{D}]}$}
\Begin{
    \tcp{Each $(\report_{i_t},\report_{j_t})$ contains two reviews of the same item.}
    Train a short text embedder based on the set of judgments $D$\;
    $N_c:= 30$\;
    Map judgments to ${1, 2, \ldots, N_c}$ using Minibatch K-means clustering with the short text embedder.\;
    Denote $\report[c]$ as a binary indicator for the existence of a judgment with label $c$ in report $\report$\;
    \BlankLine
    
    \For{$c = 1$ \KwTo $N_c$}{
    \For{$b_0 = 0$ \KwTo $1$}{
    \For{$b_1 = 0$ \KwTo $1$}{
        
        $p_{b_0,b_1}(c):=\frac{1}{|D|}\sum_{(\report_{i_t},\report_{j_t})\in D} \mathbf{1}[\report_{i_t}[c] = b_0,\report_{j_t}[c] = b_1]$\;
    }}
}
}

\BlankLine

\textbf{Query of \textsc{GPPM} score:} \Query{$\report_i, \report_j$} \;
\KwIn{textual reports $\report_i$ and $\report_j$}
\KwOut{Score for agent $i$: $\textsc{Cluster}(\report_i,\report_j)$}
\Begin{

    \tcp{$score \triangleq \log\Pr[\Signal_j = \report_j \mid \Signal_i = \report_i]$}
    $score := 0$\;
    
    \For{$c = 1$ \KwTo $N_c$}{
        $score := score + \log\frac{p_{\report_i[c],\report_j[c]}(c)}{p_{\report_i[c],0}(c) + p_{\report_i[c],1}(c)}$\;
    }

    \BlankLine
    
    $\textsc{Cluster}(\report_i,\report_j) := score$\;
}
\end{algorithm}

\begin{table}[h]
\small\centering\renewcommand{\arraystretch}{0.6}
\begin{tabular}{cccccccc}
\toprule
Dataset & Implementation & $\bar{d}$ & $\sigma(d)$ & $SE(\bar{d})$ & t-statistic & p-value & $-\log_{10}(\text{p-value})$\\ 
\midrule
ICLR & \textsc{Cluster} & 0.084 & 0.397 & 0.013 & 6.659 & 3.6e-11 & 10.438\\             
\cmidrule{1-8}  
Yelp & \textsc{Cluster} & 0.482 & 1.032 & 0.033 & 14.774 & 4.0e-45 & 44.396\\
\bottomrule
\end{tabular}
\caption{Statistics Metrics of Random Report Replacement Evaluation (\textsc{Cluster})}
\label{table:random-replace-cluster}
\end{table}

\paragraph{Results}
\Cref{table:random-replace-cluster,table:sentence-degrade-cluster} present the performance results of the \textsc{Cluster} implementation. Our results suggest that the performance of \textsc{Cluster} is dominated by \textsc{Token-Preprocess} and \textsc{Judgment}, suggesting a large space of improvement. We hypothesize that this is because the clustering step incurs a significant information loss. The development of better context-specific clustering algorithms can potentially improve the performance of \textsc{Cluster}.

\begin{table}[h]
\small\centering\renewcommand{\arraystretch}{0.6}
\begin{tabular}{cccccccc}
\toprule
Dataset & Implementation & $\bar{d}$ & $\sigma(d)$ & $SE(\bar{d})$ & t-statistic & p-value & $-\log_{10}(\text{p-value})$\\ 
\midrule
ICLR & \textsc{Cluster} & 0.015 & 0.219 & 0.007 & 2.203 & 1.4e-02 & 1.853\\             
\cmidrule{1-8}  
Yelp & \textsc{Cluster} & 0.044 & 0.474 & 0.015 & 2.955 & 1.6e-03 & 2.796\\
\bottomrule
\end{tabular}
\caption{Statistics Metrics of Judgment-Level Degradation Evaluation (\textsc{Cluster})}
\label{table:sentence-degrade-cluster}
\end{table}

\paragraph{Discussions and Limitations}
We acknowledge several limitations with the $\textsc{Cluster}$ implementation. First, it relies on historical data and thus is not zero-shot. Second, without non-trivial adaptations, $\textsc{Cluster}$ is not compatible with the \textsc{Generative Synopsis Peer Prediction Mechanism}. Lastly, it relies on certain (strong) assumptions. In particular, Assumption (1) drops a lot of information within the textual data and only captures the information about what categories of judgments the review contains. Moreover, Assumption (2) overlooks the correlations between clusters, thereby weakening the predictive power of the judgment cluster information. These limitations suggest large spaces for improvements.

%% file: appendix_prompt.tex
\section{Prompts for LLMs}\label{appendix:prompt}

This section details the prompts employed in interactions with Large Language Models (LLMs), including the specific models used, to enable the replication of our results. It is worthwhile to mention that these prompts were crafted without extensive refinement, suggesting potential for further improvement.

\subsection{Rephrasing text signals}\label{app:pre-process}

In both methods described in \Cref{sec:method}, we utilize LLMs to rephrase original text signals into a format conducive for further LLM generation or analysis. Specifically, we instruct the LLMs to reframe the text signals into a list of judgments, prefixed as required.

\paragraph{ICLR dataset} We employ the model \textsf{gpt-4-1106-preview} for this task.
\begin{tcolorbox}[enhanced, colback=gray!7, frame hidden]\begin{scriptsize}
\begin{normalsize}\textbf{System Prompt}\end{normalsize}\smallskip

Carefully read the text of a scientific paper review. You should summarize each evaluation in the review in a separate line. Begin each summary line with one of the following phrases: `The reviewer appreciates', `The reviewer criticizes', `The reviewer questions', `The reviewer suggests'.  You need to keep the summary as concise as possible, excluding specific details about the paper's content, such as topics, ideas, methods, findings, and any mathematical symbols.

You should ensure that even if multiple evaluations are mentioned in the same sentence in the original review, you should still split it into separate lines. For example, you should not output a line like `The reviewer appreciates the well-written paper and good experimental performance'. In contrast, you should output `The reviewer appreciates the well-written paper' and `The reviewer appreciates good experimental performance' in two lines.

\bigskip\begin{normalsize}\textbf{User Prompt}\end{normalsize}\smallskip

\texttt{\{Report\_of\_the\_agent\_i\}}

\end{scriptsize}\end{tcolorbox}

\begin{tcolorbox}[enhanced, colback=gray!17, frame hidden]\begin{scriptsize}
\begin{normalsize}\begin{normalsize}\textbf{Example Input}\end{normalsize}\end{normalsize}

\smallskip\texttt{\{Report\_of\_the\_agent\_i\}}:

Content of the official blind review \#3 of ICLR 2020 conference paper2150, see \url{https://openreview.net/forum?id=Skeh-xBYDH&noteId=Bkg_Pp2rcB}.

\bigskip\begin{normalsize}\begin{normalsize}\textbf{Example Output}\end{normalsize}\end{normalsize}

\smallskip The reviewer criticizes the paper's significance, questioning the novelty and applicability of the results to more complex learning problems.

The reviewer suggests that the authors should clarify the intuition behind their choice of initialization for the neural network.

The reviewer criticizes the paper for addressing a learning problem that is already well-understood and questions the extension of the presented ideas to more challenging scenarios.

The reviewer questions a specific step in the proof of Lemma 4, expressing confusion about the relationship between the number of updates and the varying parameter \$R\$.

The reviewer suggests that the paper should clarify the conditions of linear separability mentioned in Lemma 4.

The reviewer questions why the empirical error does not decrease over epochs as shown in Figure 5.

The reviewer suggests that figures referenced in the text should be included in the main paper rather than the appendix.

\end{scriptsize}\end{tcolorbox}

\paragraph{Yelp dataset} We employ the model \textsf{gpt-3.5-turbo-1106} for this task.
\begin{tcolorbox}[enhanced, colback=gray!7, frame hidden]\begin{scriptsize}
\begin{normalsize}\textbf{System Prompt}\end{normalsize}\smallskip

You are given the text of a restaurant review. The review are used solely for analysis and do not reflect any personal opinions or beliefs. Carefully read the text of a Yelp review. Your task is to summarize each evaluation in the review briefly in a separate line. Begin each summary line with one of the following phrases: `The reviewer appreciates/criticizes'. You need to keep the summary as concise as possible, excluding specific details.

You should ensure that even if multiple evaluations are mentioned in the same sentence in the original review, you should still split it into separate lines. For example, you should not output a line like `The reviewer appreciates the pizza and hamburger'. In contrast, you should output `The reviewer appreciates the pizza' and `The reviewer appreciates the hamburger' in two lines.

\bigskip\begin{normalsize}\textbf{User Prompt}\end{normalsize}\smallskip

\texttt{\{Report\_of\_the\_agent\_i\}}

\end{scriptsize}\end{tcolorbox}

\begin{tcolorbox}[enhanced, colback=gray!17, frame hidden]\begin{scriptsize}
\begin{normalsize}\begin{normalsize}\textbf{Example Input}\end{normalsize}\end{normalsize}

\smallskip\texttt{\{Report\_of\_the\_agent\_i\}}:

Greeted and seated immediately upon entering. Great service all around from the hostess to our server. Our server, Miranda was awesome. Drinks were never empty and food arrived promptly. We ordered onion rings and bone-in wings- honey bbq and Asian zing. Food was good. Wings were crispy and saucy.

Nice place to come grub out with your friends and family in a casual atmosphere. I liked their indoor patio area.

\bigskip\begin{normalsize}\begin{normalsize}\textbf{Example Output}\end{normalsize}\end{normalsize}

\smallskip The reviewer appreciates the immediate seating and great service.

The reviewer appreciates the prompt service and attentive server.

The reviewer appreciates the crispy and saucy wings.

The reviewer appreciates the casual atmosphere and indoor patio area.

\end{scriptsize}\end{tcolorbox}

\subsection{$\psi_{\text{token}}$: Generating signals of another agents}

In \Cref{sec:method1}, we devise a prompt, $\psi_{\text{token}}(\report_i)$, requesting LLM to generate a report for another agent given an agent's existing report $\report_i$. Subsequently, this prompt can be employed to query the LLM for $\Pr_{\operatorname{LLM}(\psi_{\text{token}}(\report_i))}[\Signal=\report_j]$, serving as a reliable approximation for $\Pr[\Signal_j=\report_j | \Signal_i=\report_i]$ through \textit{logprob} feedback.

\paragraph{ICLR dataset} We employ the model \textsf{Llama-2-70b-chat} for this task.
\begin{tcolorbox}[enhanced, colback=gray!7, frame hidden]\begin{scriptsize}
\begin{normalsize}\textbf{System Prompt}\end{normalsize}\smallskip

You are the second reviewer for a scientific paper. You are given the abstract of the paper and a list of review judgments from the first reviewer, starting with `The reviewer appreciates/criticizes/questions/suggests'. Your task is to provide your own judgments of the paper based on the given materials. You should create a separate line for each judgment you have, starting with `The reviewer appreciates/criticizes/questions/suggests'. Ensure your judgments are concise, excluding specific details about the paper's content.

\bigskip\begin{normalsize}\textbf{User Prompt}\end{normalsize}\smallskip

[Abstract of the paper]

\texttt{\{Abstract\_of\_the\_paper\} if mechanism == GSPPM else "Not Provided"}

[Review judgments from the first reviewer]

\texttt{\{Rephrased\_report\_of\_the\_agent\_i\}}

\bigskip\begin{normalsize}\textbf{Forced LLM Output}\end{normalsize}

\texttt{\{Rephrased\_report\_of\_the\_agent\_j\}}

\end{scriptsize}\end{tcolorbox}
\paragraph{Yelp dataset} We employ the model \textsf{Llama-2-70b-chat} for this task.
\begin{tcolorbox}[enhanced, colback=gray!7, frame hidden]\begin{scriptsize}
\begin{normalsize}\textbf{System Prompt}\end{normalsize}\smallskip

You are the second food critic for a restaurant. You are given the categories of the restaurant and a list of review judgments from the first food critic, starting with `The reviewer appreciates/criticizes/questions/suggests'. Your task is to provide your own judgments of the restaurant based on the given materials. You should create a separate line for each judgment you have, starting with `The reviewer appreciates/criticizes/questions/suggests'. Ensure your judgments are concise.

\bigskip\begin{normalsize}\textbf{User Prompt}\end{normalsize}\smallskip

[Categories of the restaurant]

\texttt{\{Categories\_of\_the\_restaurant\} if mechanism == GSPPM else "Not Provided"}

[Review judgments from the first food critic]

\texttt{\{Rephrased\_report\_of\_the\_agent\_i\}}

\bigskip\begin{normalsize}\textbf{Forced LLM Output}\end{normalsize}

\texttt{\{Rephrased\_report\_of\_the\_agent\_j\}}

\end{scriptsize}\end{tcolorbox}

\subsection{$\psi_{\text{judg}}$: predicting judgments of another agents}
\label{app:prompt_judge}

In \Cref{sec:method2}, we devise a prompt, $\psi_{\text{judg}}(x_i,w)$, requesting LLM to predict how the probability of another agent making a judgment $w$ shifts in response to an agent's signal $x_i$. In practice, the prompt is structured to obtain the LLM's predictions for all judgments $w$ within $x_j$ in a single response.

\paragraph{ICLR dataset} We employ the model \textsf{gpt-4-1106-preview} for this task.
\begin{tcolorbox}[enhanced, colback=gray!7, frame hidden]\begin{scriptsize}
\begin{normalsize}\textbf{System Prompt}\end{normalsize}\smallskip

You will receive the abstract of a scientific paper and a set of initial reviewer comments on this paper, with each comment presented on a separate line. Following these, you will be provided with a list of additional, specific reviewer comments.

Your task is to evaluate how each of the additional comments aligns with the initial set of reviewer comments. For each additional comment, assign a score from -3 to 3, indicating the degree of support or contradiction by the initial comments. 

For each additional comment, you should output a single line, including the reasons first (starts with `the initial comments appreciate/criticizes/do not mention'), and then your score (in the format of <SCORE=...>).

You should use the following scale for scoring:

-3: Strong contradiction - The initial comments directly oppose the additional comment, significantly reducing its credibility (likelihood less than 1/8). Example: Initial comments appreciate the writing, but the additional comment criticizes it.

-2: Moderate contradiction - The initial comments indirectly oppose the additional comment, moderately reducing its credibility (likelihood about 1/4). Example: Initial comments commend the theoretical analysis, but the additional comment faults a theorem proof.

-1: Slight contradiction - The initial comments slightly reduce the credibility of the additional comment (likelihood about 1/2).

0: Neutral - The initial comments neither support nor contradict the additional comment, leaving its credibility unchanged. Please try to avoid giving a score of 0 unless you cannot find any potential support or contradiction.

1: Slight support - The initial comments slightly increase the credibility of the additional comment (likelihood about 2).

2: Moderate support - The initial comments indirectly support the additional comment, moderately increasing its credibility (likelihood about 4). Example: Initial comments commend the theoretical analysis, and the additional comment finds the proofs are enlightening.

3: Strong support - The initial comments directly support the additional comment, significantly increasing its credibility (likelihood more than 8 times). Example: Both the initial comments and the additional comment appreciate the writing quality.

\bigskip\begin{normalsize}\textbf{User Prompt}\end{normalsize}\smallskip

[Abstract of the paper]

\texttt{\{Abstract\_of\_the\_paper\} if mechanism == GSPPM else "Not Provided"}

[Initial reviewer comments]

\texttt{\{Rephrased\_report\_of\_the\_agent\_i\}}

[Additional reviewer comments]

\texttt{\{Rephrased\_report\_of\_the\_agent\_j\}}

\end{scriptsize}\end{tcolorbox}

\begin{tcolorbox}[enhanced, colback=gray!17, frame hidden]\begin{scriptsize}

\begin{normalsize}\begin{normalsize}\textbf{Example Input}\end{normalsize}\end{normalsize}

\smallskip\texttt{\{Abstract\_of\_the\_paper\}}:

This work provides an additional step in the theoretical understanding of neural networks. We consider neural networks with one hidden layer and show that when learning symmetric functions, one can choose initial conditions so that standard SGD training efficiently produces generalization guarantees. We empirically verify this and show that this does not hold when the initial conditions are chosen at random. The proof of convergence investigates the interaction between the two layers of the network. Our results highlight the importance of using symmetry in the design of neural networks.

\smallskip\texttt{\{Rephrased\_report\_of\_the\_agent\_i\}}:

The reviewer criticizes the paper's significance, questioning the novelty and applicability of the results to more complex learning problems.

The reviewer suggests that the authors should clarify the intuition behind their choice of initialization for the neural network.

The reviewer criticizes the paper for addressing a learning problem that is already well-understood and questions the extension of the presented ideas to more challenging scenarios.

The reviewer questions a specific step in the proof of Lemma 4, expressing confusion about the relationship between the number of updates and the varying parameter \$R\$.

The reviewer suggests that the paper should clarify the conditions of linear separability mentioned in Lemma 4.

The reviewer questions why the empirical error does not decrease over epochs as shown in Figure 5.

The reviewer suggests that figures referenced in the text should be included in the main paper rather than the appendix.

\smallskip\texttt{\{Rephrased\_report\_of\_the\_agent\_j\}}:

The reviewer criticizes the lack of motivation for the underlying problem, suggesting that the class of symmetric boolean functions is not challenging enough to warrant the use of neural networks.
The reviewer questions the necessity of using a neural network for the problem when simpler methods are available.

The reviewer suggests that the problem would be more interesting if it focused on symmetric functions in the real domain.

The reviewer criticizes the writing, specifically the need for more detailed proofs and better motivation in the introduction.

The reviewer questions the choice of representation for indicators using ReLUs and suggests a simpler alternative from prior work.

The reviewer criticizes the experimental plots for being hard to parse and inconsistent.

The reviewer suggests using line plots instead of scatter plots to better highlight trends.

The reviewer criticizes the inconsistent sampling on the x-axis and varying number of epochs in experiments.

The reviewer suggests that important plots should be included in the main paper and questions whether the experiments were repeated multiple times or based on a single run.

\bigskip\begin{normalsize}\begin{normalsize}\textbf{Example Output}\end{normalsize}\end{normalsize}

\smallskip[1] The initial comments criticize the paper for addressing a learning problem that is already well-understood and question the extension of the presented ideas to more challenging scenarios. This aligns with the additional comment criticizing the lack of motivation for the underlying problem, suggesting that the class of symmetric boolean functions is not challenging enough. <SCORE=3>

[2] The initial comments do not mention the availability of simpler methods for the problem, but the criticism of the paper's significance and the questioning of the novelty and applicability of the results indirectly support the additional comment questioning the necessity of using a neural network. <SCORE=2>

[3] The initial comments do not mention the domain of the symmetric functions, so there is no direct support or contradiction for the additional comment suggesting that the problem would be more interesting if it focused on symmetric functions in the real domain. <SCORE=0>

[4] The initial comments suggest that the authors should clarify the intuition behind their choice of initialization and clarify the conditions of linear separability mentioned in Lemma 4. This indirectly supports the additional comment criticizing the writing and the need for more detailed proofs and better motivation. <SCORE=2>

[5] The initial comments do not mention the representation for indicators using ReLUs, so there is no direct support or contradiction for the additional comment suggesting a simpler alternative from prior work. <SCORE=0>

[6] The initial comments criticize the empirical error presentation in Figure 5 and suggest that figures should be included in the main paper. This supports the additional comment criticizing the experimental plots for being hard to parse and inconsistent. <SCORE=3>

[7] The initial comments do not directly address the format of the plots, but the suggestion that figures should be included in the main paper and the criticism of the empirical error presentation indirectly support the additional comment suggesting the use of line plots instead of scatter plots. <SCORE=2>

[8] The initial comments do not mention the sampling on the x-axis or the varying number of epochs in experiments, so there is no direct support or contradiction for the additional comment criticizing the inconsistent sampling and varying number of epochs. <SCORE=0>

[9] The initial comments suggest that figures should be included in the main paper, which directly supports the additional comment that important plots should be included in the main paper. The initial comments do not address whether the experiments were repeated multiple times, so there is no direct support or contradiction for that part of the additional comment. <SCORE=2>

\end{scriptsize}\end{tcolorbox}

\paragraph{Yelp dataset} We employ the model \textsf{gpt-4-1106-preview} for this task.
\begin{tcolorbox}[enhanced, colback=gray!7, frame hidden]\begin{scriptsize}
\begin{normalsize}\textbf{System Prompt}\end{normalsize}\smallskip

You will receive the categories of a restaurant and a set of initial reviewer comments on this restaurant, with each comment presented on a separate line. Following these, you will be provided with a list of additional, specific reviewer comments.

Your task is to evaluate how each of the additional comments aligns with the initial set of reviewer comments. For each additional comment, assign a score from -3 to 3, indicating the degree of support or contradiction by the initial comments. 

For each additional comment, you should output a single line, including the reasons first (starts with `the initial comments appreciate/criticizes/do not mention'), and then your score (in the format of <SCORE=...>).

You should use the following scale for scoring:

-3: Strong contradiction - The initial comments directly oppose the additional comment, significantly reducing its credibility (likelihood less than 1/8). Example: Initial comments appreciate the pizza, but the additional comment criticizes it.

-2: Moderate contradiction - The initial comments indirectly oppose the additional comment, moderately reducing its credibility (likelihood about 1/4). Example: Initial comments appreciate the overall taste, but the additional comment criticizes the chicken.

-1: Slight contradiction - The initial comments slightly reduce the credibility of the additional comment (likelihood about 1/2).

0: Neutral - The initial comments neither support nor contradict the additional comment, leaving its credibility unchanged. Please try to avoid giving a score of 0 unless you cannot find any potential support or contradiction.

1: Slight support - The initial comments slightly increase the credibility of the additional comment (likelihood about 2).

2: Moderate support - The initial comments indirectly support the additional comment, moderately increasing its credibility (likelihood about 4). Example: Initial comments appreciate the overall taste, and the additional comment appreciates the chicken.

3: Strong support - The initial comments directly support the additional comment, significantly increasing its credibility (likelihood more than 8 times). Example: Both the initial comments and the additional comment appreciate the pizza.

\bigskip\begin{normalsize}\textbf{User Prompt}\end{normalsize}\smallskip

[Categories of the restaurant]

\texttt{\{Categories\_of\_the\_restaurant\} if mechanism == GSPPM else "Not Provided"}

[Initial reviewer comments]

\texttt{\{Rephrased\_report\_of\_the\_agent\_i\}}

[Additional reviewer comments]

\texttt{\{Rephrased\_report\_of\_the\_agent\_j\}}

\end{scriptsize}\end{tcolorbox}

\begin{tcolorbox}[enhanced, colback=gray!17, frame hidden]\begin{scriptsize}

\begin{normalsize}\begin{normalsize}\textbf{Example Input}\end{normalsize}\end{normalsize}

\smallskip\texttt{\{Categories\_of\_the\_restaurant\}}:

Not Provided

\smallskip\texttt{\{Rephrased\_report\_of\_the\_agent\_i\}}:

The reviewer appreciates the immediate seating and great service.

The reviewer appreciates the prompt service and attentive server.

The reviewer appreciates the crispy and saucy wings.

The reviewer appreciates the casual atmosphere and indoor patio area.

\smallskip\texttt{\{Rephrased\_report\_of\_the\_agent\_j\}}:

The reviewer appreciates the cleanliness of the restaurant.

The reviewer appreciates the awesome staff and prompt service.

The reviewer criticizes the flavor of the mango habanero wings.

\bigskip\begin{normalsize}\begin{normalsize}\textbf{Example Output}\end{normalsize}\end{normalsize}

\smallskip [1] The initial comments do not mention the cleanliness of the restaurant. <SCORE=0>

[2] The initial comments appreciate the great service and attentive server, which aligns with appreciating the awesome staff and prompt service. <SCORE=3>

[3] The initial comments appreciate the crispy and saucy wings, but do not specifically mention the mango habanero flavor. This is a slight contradiction since the initial comments were positive about the wings in general. <SCORE=-1>

\end{scriptsize}\end{tcolorbox}

\subsection{Generating review from paper}
\label{app:prompt_review}

In \Cref{sec:result2}, we employ LLMs to generate signals in the peer review scenario, demonstrating that the \textsc{Generative Synopsis Peer Prediction Mechanism (GSPPM)} is more effective at eliciting high-effort signals than the \textsc{Generative Peer Prediction Mechanism (GPPM)}. We design a prompt directing the LLM to act as a reviewer assessing a particular paper.

\paragraph{ICLR dataset} We employ the model \textsf{gpt-4-1106-preview} and \textsf{gpt-3.5-turbo-1106} for this task.
\begin{tcolorbox}[enhanced, colback=gray!7, frame hidden]\begin{scriptsize}
\begin{normalsize}\textbf{System Prompt}\end{normalsize}\smallskip

You are a reviewer for a machine learning conference. Your task is to review a submitted paper. In your review, you should thoroughly analyze the paper and offer detailed feedback. Highlight the paper's strengths and weaknesses. Your review should be objective, critical, and supportive, aimed at helping the authors improve their work and contributing to the advancement of machine learning research.

\bigskip\begin{normalsize}\textbf{User Prompt}\end{normalsize}\smallskip

\texttt{\{Paper\_content\}}

\end{scriptsize}\end{tcolorbox}

\begin{tcolorbox}[enhanced, colback=gray!17, frame hidden]\begin{scriptsize}

\begin{normalsize}\begin{normalsize}\textbf{Example Input}\end{normalsize}\end{normalsize}

\smallskip\texttt{\{Paper\_content\}}:

Content of the ICLR 2020 conference paper2150, see \url{https://openreview.net/forum?id=Skeh-xBYDH}.

\bigskip\begin{normalsize}\begin{normalsize}\textbf{Example Output}\end{normalsize}\end{normalsize}

\smallskip**Review of "On Symmetry and Initialization for Neural Networks"**

**Summary:**

The paper presents a theoretical analysis of the role of symmetry in the initialization of neural networks with one hidden layer. The authors focus on symmetric functions and demonstrate that with symmetry-based initialization, standard Stochastic Gradient Descent (SGD) can efficiently converge to zero empirical error and provide generalization guarantees. The paper includes empirical validation of the theoretical results and explores the impact of initialization on the learning of symmetric functions, including the parity function.

**Strengths:**

1. **Theoretical Contribution:** The paper addresses an important aspect of neural network initialization and provides a novel theoretical framework for understanding the impact of symmetry on learning and generalization.

2. **Empirical Validation:** The authors support their theoretical claims with empirical experiments, which is commendable as it provides a practical perspective on the theoretical results.

3. **Relevance to the Field:** The topic of initialization is highly relevant to the machine learning community, and the paper's focus on symmetry could inspire further research in this area.

4. **Clarity of Presentation:** The paper is well-structured, with clear section titles and a logical flow that guides the reader through the theoretical development and empirical validation.

**Weaknesses:**

1. **Limited Scope:** The paper focuses exclusively on neural networks with one hidden layer. While this is a reasonable simplification for theoretical analysis, it limits the applicability of the results to more complex architectures that are commonly used in practice.

2. **Generalization of Results:** The generalization guarantees provided are specific to the class of symmetric functions. It is unclear how these results extend to other classes of functions or to more realistic, non-symmetric data distributions.

3. **Experimental Details:** The paper could benefit from more detailed descriptions of the experimental setup, including the network architectures used, hyperparameters, and the datasets. This would improve the reproducibility of the results.

4. **Comparison with Other Initialization Techniques:** The paper does not compare the proposed symmetry-based initialization with other popular initialization methods, such as Xavier or He initialization, which would be valuable for contextualizing the results.

**Suggestions for Improvement:**

1. **Extend to More Complex Architectures:** Future work could explore how symmetry-based initialization impacts the learning dynamics of deeper and more complex neural networks.

2. **Broader Empirical Evaluation:** The authors should consider evaluating their approach on a wider range of functions and datasets, including non-symmetric cases, to assess the robustness of their findings.

3. **Detailed Experimental Methodology:** Providing more details about the experimental methodology would enhance the paper's credibility and allow others to replicate the findings.

4. **Benchmarking:** A comparison with other initialization methods would provide a more comprehensive understanding of the proposed approach's strengths and limitations.

**Overall Evaluation:**

The paper makes a valuable theoretical contribution to the understanding of neural network initialization, particularly in the context of symmetric functions. The empirical validation is a strong point, although the experiments could be expanded to provide a more comprehensive evaluation. The paper's focus on a single hidden layer network is both a strength, in terms of theoretical tractability, and a weakness, in terms of practical relevance. To move forward, the authors should consider extending their analysis to more complex architectures and providing a broader empirical evaluation.

**Recommendation:**

Accept with minor revisions. The paper presents interesting theoretical insights and empirical results that contribute to the understanding of neural network initialization. However, the authors should address the limitations regarding the scope and experimental details to strengthen the paper's impact.

\end{scriptsize}\end{tcolorbox}